\documentclass[sigconf]{acmart}

\copyrightyear{2026}
\acmYear{2026}
\setcopyright{cc}
\setcctype{by}
\acmConference[MM '26] {Proceedings of the 34th ACM International Conference on Multimedia}{November 10--14, 2026}{Rio de Janeiro, Brazil.}
\acmBooktitle{Proceedings of the 34th ACM International Conference on Multimedia (MM '26), November 10--14, 2026, Rio de Janeiro, Brazil}
\acmISBN{979-8-4007-2213-4/2026/11}
\acmDOI{10.1145/3767308.3835472}

\usepackage{multirow}

\newtheorem{lemma}{Lemma}

\newtheorem{assumption}{Assumption}
\newtheorem{proposition}{Proposition}
\usepackage{algorithm}
\usepackage{algorithmic}
\AtBeginDocument{%
  }

\begin{document}

\title{Stabilizing Multi-Attack Adversarial Training via Bandit Optimization}

\author{Rui Wang}
\authornote{Equal contribution.}
\affiliation{%
  \institution{Peking University}
  \city{Beijing}
  \country{China}
}
\email{2300010607@stu.pku.edu.cn}

\author{Zeming Wei}
\authornotemark[1]
\affiliation{%
  \institution{Peking University}
  \city{Beijing}
  \country{China}
}
\email{weizeming@stu.pku.edu.cn}

\author{Xiyue Zhang}
\affiliation{%
  \institution{University of Bristol}
  \city{Bristol}
  \country{United Kingdom}
  }
\email{xiyue.zhang@bristol.ac.uk}

\author{Meng Sun}
\authornote{Corresponding author.}
\affiliation{%
  \institution{Peking University}
  \city{Beijing}
  \country{China}
  }
\email{sunm@pku.edu.cn}

\renewcommand{\shortauthors}{Wang et al.}

\ccsdesc[500]{Computing methodologies~Machine learning}
\ccsdesc[300]{Security and privacy~Software and application security}
\ccsdesc[300]{Computing methodologies~Machine learning approaches}

\keywords{Robustness, Adversarial Training, Optimization, Bandit Algorithms}



\begin{abstract}
Deep Neural Networks (DNNs) remain vulnerable to diverse adversarial perturbations, motivating multi-attack adversarial training (AT) for improved robustness. However, existing methods either incur prohibitive overhead by computing all attacks at each iteration, or rely on stochastic sampling over adversarial examples, which may cause excessive parameter drift. To address these issues, we propose Calibrated Adversarial Sampling (CAS), an efficient and stable framework that reformulates multi-attack AT as a multi-armed bandit optimization problem. By sampling a single attack per iteration that dynamically balances exploration and exploitation, CAS significantly reduces training cost while mitigating optimization conflicts across attacks and controlling excessive parameter drifts. Extensive experiments demonstrate that CAS achieves superior overall robustness at low computational cost, offering a scalable and principled approach to robust generalization against multi-attack settings. Our code is available at \url{https://github.com/1240148048/CAS}.
\end{abstract}    
\maketitle

\section{Introduction}
\label{sec:intro}

Adversarial attacks~\cite{szegedy2013intriguing} have posed significant vulnerabilities to Deep Neural Networks (DNNs), where adversaries can add imperceptible~\cite{goodfellow2014explaining,croce2020reliable} or semantic~\cite{hendrycks2019benchmarking,hsiung2023towards} perturbations to craft adversarial examples that lead the target DNN to make incorrect predictions. So far, the existence of adversarial examples has raised serious concerns about the reliability of DNNs~\cite{liu2022trustworthy,li2023trustworthy,liu2023towards}, compromising their safe and trustworthy deployment in real-world scenarios.
\begin{figure}
    \centering
    \begin{tabular}{ccc}
    \includegraphics[width=0.3\linewidth]{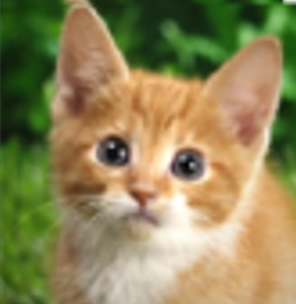}&
    \includegraphics[width=0.3\linewidth]{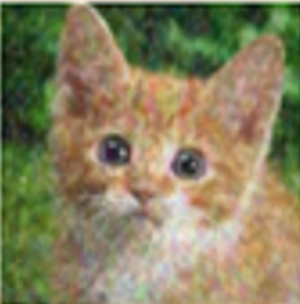}&
    \includegraphics[width=0.3\linewidth]{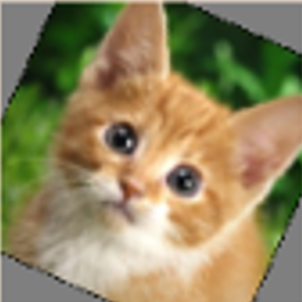}
    \\
    (a) Original & (b) $\ell_p$-norm & (c) Semantic
    \end{tabular}
    \vspace{-10pt}
    \caption{Illustration of adversarial perturbations by different attack types from~\cite{hsiung2023towards}.}
    \Description{Illustration of adversarial perturbations by different attack types from~\cite{hsiung2023towards}.}
    \label{fig:example-ae}
    \vspace{-10pt}
\end{figure}
While adversarial training (AT) has emerged as one of the most effective defenses against adversarial perturbations, it is typically formulated to optimize robustness against a single attack model. Specifically, AT solves a min--max optimization problem:
\begin{equation}
\min_{\theta} \; \mathbb{E}_{(x,y)\sim \mathcal{D}} \left[ \max_{\delta \in \mathcal{S}} \mathcal{L}(f_{\theta}(x+\delta), y) \right],
\end{equation}
where $\mathcal{S}$ defines the threat model. In practice, most existing methods instantiate $\mathcal{S}$ using a specific attack type (\textit{e.g.}, an $\ell_p$-bounded perturbation set), which inherently limits their robustness to that particular class of perturbations and hinders generalization to diverse and unforeseen adversarial conditions.

\begin{figure*}
    \centering
    \begin{tabular}{ccc}
    \includegraphics[width=0.28\linewidth]{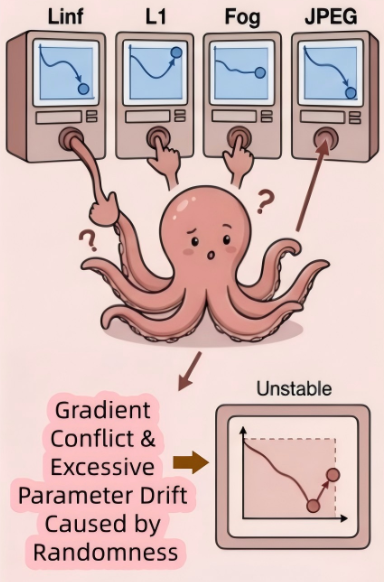}
         & 
    \includegraphics[width=0.28\linewidth]{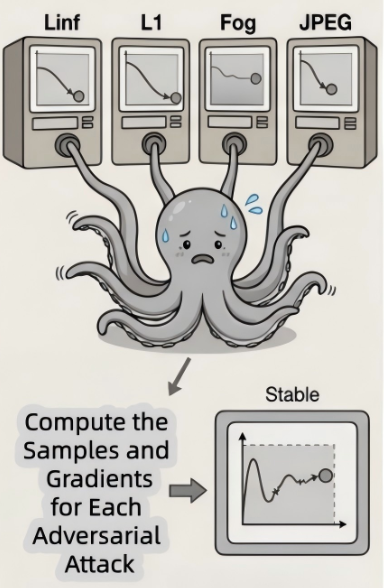}     
         &
    \includegraphics[width=0.34\linewidth]{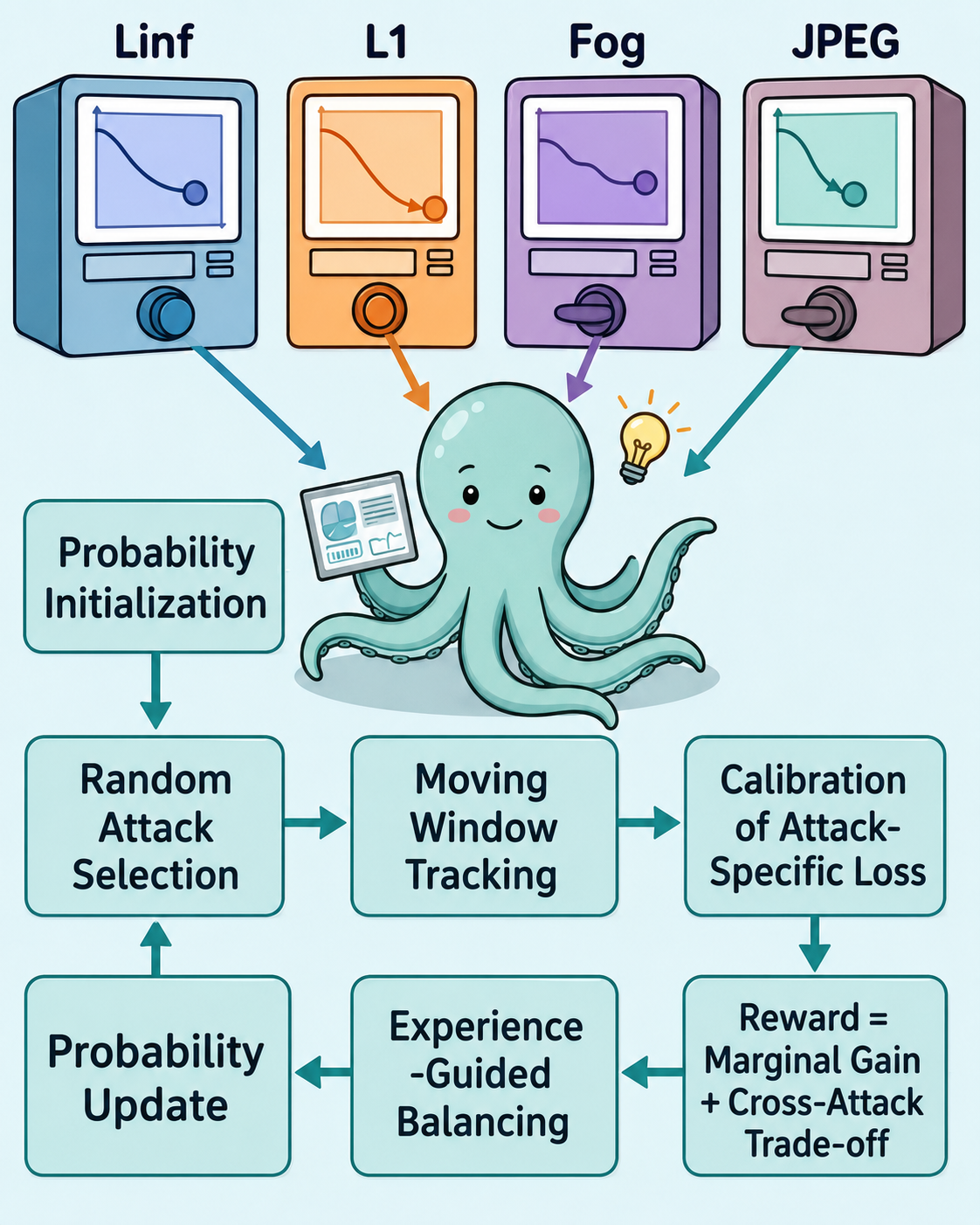}
         \\
    (a) Naive Random Sampling
         & 
    (b) Multi-task Optimization
        &
    (c) CAS (calibrated Adversarial Sampling)
    \end{tabular}
    \vspace{-10pt}
    \caption{
    A comparison between CAS and baselines. (a) Naive random sampling is efficient, but unstable; (b) Multi-task optimization is stable, but time-consuming; (c) Our CAS framework achieves both efficiency and stability.
    }
    \label{fig:pipeline}
    \Description{pipeline}
    \vspace{-10pt}
\end{figure*}

However, in real-world vision applications such as image and video analysis, DNNs must operate on visual data that undergo complex processing pipelines, during which diverse adversarial perturbations can arise. For example, models may be simultaneously exposed to pixel-level $\ell_p$ attacks as well as common distortions and transformations such as JPEG compression artifacts, Gaussian blur, resizing, and weather-induced degradations (\textit{e.g.}, fog or rain). These perturbations, which can be roughly grouped into \textbf{$\ell_p$-bounded} and \textbf{semantic} types, exhibit substantially different behaviors and impacts on model predictions. Consequently, improving robustness across multiple attack types simultaneously has become an important yet challenging problem.

A natural extension of AT is to incorporate multiple attack types into the optimization process. In such approaches, each training iteration requires generating adversarial examples and obtaining the gradients with respect to each attack. These gradients are then combined either by direct weighted averaging~\cite{tramer2019adversarial}, or through gradient projection techniques to resolve conflicting directions before performing the parameter update~\cite{yu2020gradient}. However, this line of methods suffers from two fundamental limitations. First, it incurs prohibitive computational overhead, as the computation need for all attack types at every iteration leads to costs that scale linearly with the number of attacks. Second, these approaches implicitly treat different attack objectives in a uniform manner, overlooking the fact that adversarial losses can differ substantially in scale and optimization difficulty across attack types. This mismatch often results in biased updates, where certain attacks dominate the optimization while others are not addressed adequately.

To improve efficiency, an alternative is to sample a single attack per iteration~\cite{madaan2021learning,croce2022adversarial}. While computationally appealing, such stochastic optimization strategies rely on randomly sampled adversarial examples and may induce excessive parameter drift, especially when the same attack is occasionally sampled repeatedly,
which may drive model parameters away from a desirable region that preserves robustness across multiple attack types. Moreover, naive sampling schemes fail to capture the complex and dynamic trade-offs among different robustness dimensions, leading to unstable training and suboptimal generalization.

This efficiency-stability dilemma in multi-attack AT motivates us to reconsider it as a bandit optimization problem.
Specifically, we reformulate the problem as a sequential decision-making task under partial feedback, where each attack corresponds to an arm and the learner must balance competing robustness objectives. Based on this formulation, we propose \textbf{Calibrated Adversarial Sampling (CAS)}, a fine-tuning framework for multi-attack AT. CAS normalizes attack-specific losses and maintains per-attack moving windows to estimate their marginal contributions and cross-robustness trade-offs. It further leverages a multi-armed bandit strategy to balance exploration and exploitation, reducing training bias and limiting excessive parameter drift. As a result, our approach provides a more efficient and stable solution for achieving generalized robustness under multi-attack settings.

Our main contributions are summarized as follows:
\begin{itemize}
    \item \textbf{Method.} We propose Calibrated Adversarial Sampling (CAS), a fine-tuning framework for accelerating and stabilizing multi-attack adversarial training. To our knowledge, CAS is the first multi-attack AT work using bandit optimization and quantifying the cross-attack interactions.
    
    \item \textbf{Theory.} We provide theoretical analysis of multi-attack adversarial training in a simplified setting, deriving a safe threshold for parameter drift. We further establish the convergence of CAS under stochastic attack selection.
    
    \item \textbf{Empirics.} Extensive experiments on benchmark datasets from RobustBench~\cite{li2023oodrobustbench} demonstrate that CAS consistently improves overall robustness against multiple attack types while maintaining high clean accuracy. Ablation studies further validate the effectiveness of each component in CAS.
\end{itemize}
\section{Preliminaries}
\subsection{Comprehensive Consideration of Adversarial Attacks}
\label{subsec: Unforeseen Adversarial Attacks}

While significant breakthroughs have currently been achieved in research on $\ell_p$ robustness, such as AutoAttack~\cite{croce2020reliable}, systematic investigations into semantic adversarial attacks remain comparably under-explored~\cite{croce2020reliable,liu2023towards}. Moreover, current robustness evaluations are often confined to single attack types, lacking a holistic framework that jointly considers diverse threat models. This narrow focus poses significant limitations for real-world AI deployments (\textit{e.g.}, multimedia applications such as autonomous driving and face recognition), where robustness against physically plausible and naturalistic corruptions is critical for safety-critical applications~\cite{hendrycks2019benchmarking,wei2024physical, hsiung2023towards}.

To enhance the versatility and comprehensiveness of robust fine-tuning methods in real-world applications, this work considers a broad and diverse set of adversarial attacks. From the framework proposed by~\cite{kang2019testing}, we select 17 representative semantic attacks for evaluation, including 
\textit{Wood, Elastic, Pixel, Snow, Gabor, JPEG, Glitch, Kaleidoscope, Blur, Edge, Fog, Texture, Prison, Whirlpool, Polkadot, Klotski, and Hsv}. These cover a wide range of environmental and digital perturbations that commonly arise in real-world settings. In addition, we incorporate the PerceptivePGD (PPGD) attack~\cite{zhang2018unreasonable}, a human-perception-guided method that explicitly models semantically meaningful visual changes, along with three standard  $\ell_p$-norm attacks: $\ell_\infty$, $\ell_2$, and $\ell_1$. Together, this suite comprises 21 distinct adversarial attacks, enabling a comprehensive evaluation of robustness against both semantic and norm-constrained perturbations.

\subsection{Quantifying Robustness Conflicts}

Since our research comprehensively considers 21 adversarial attacks to transcend the limitations of AT against single adversarial attack, integrating diverse adversarial attacks into a unified AT framework introduces significant challenges. In particular, the Mutually Exclusive Perturbations theory (MEPs)~\cite{kang2019testing} represents a fundamental limitation. MEPs occur when the constraint sets of two perturbation types are inherently incompatible, such that improving robustness against one attack inevitably degrades robustness against the other under fixed optimization conditions. Classic examples of this contradiction include the $\ell_p$ attack and the rotation-translation transformation as discussed in~\cite{kang2019testing}.

While the concept of MEPs qualitatively captures the inherent conflicts in achieving multi-robustness, we further conduct a quantitative study based on the pre-trained $\ell_\infty$-robust model~\cite{croce2022adversarial} on the CIFAR-10 dataset~\cite{krizhevsky2009learning}, since $\ell_\infty$ pre-trained model exhibits robustness against a wide range of adversarial attacks~\cite{madry2017towards,kotyan2019adversarial}. We choose $\ell_\infty$, $\ell_2$, $\ell_1$, and other semantic attacks. For each attack type, we sequentially perform individual fine-tuning for 3 epochs and measure the robust accuracy against all these attacks before/after this fine-tuning. As illustrated in Figure ~\ref{fig:tradeoff}, we observe several notable patterns. More complete results and analysis of this preliminary experiment can be found in Appendix~D.

\begin{figure}
    \centering
    \includegraphics[width=1.0\linewidth]{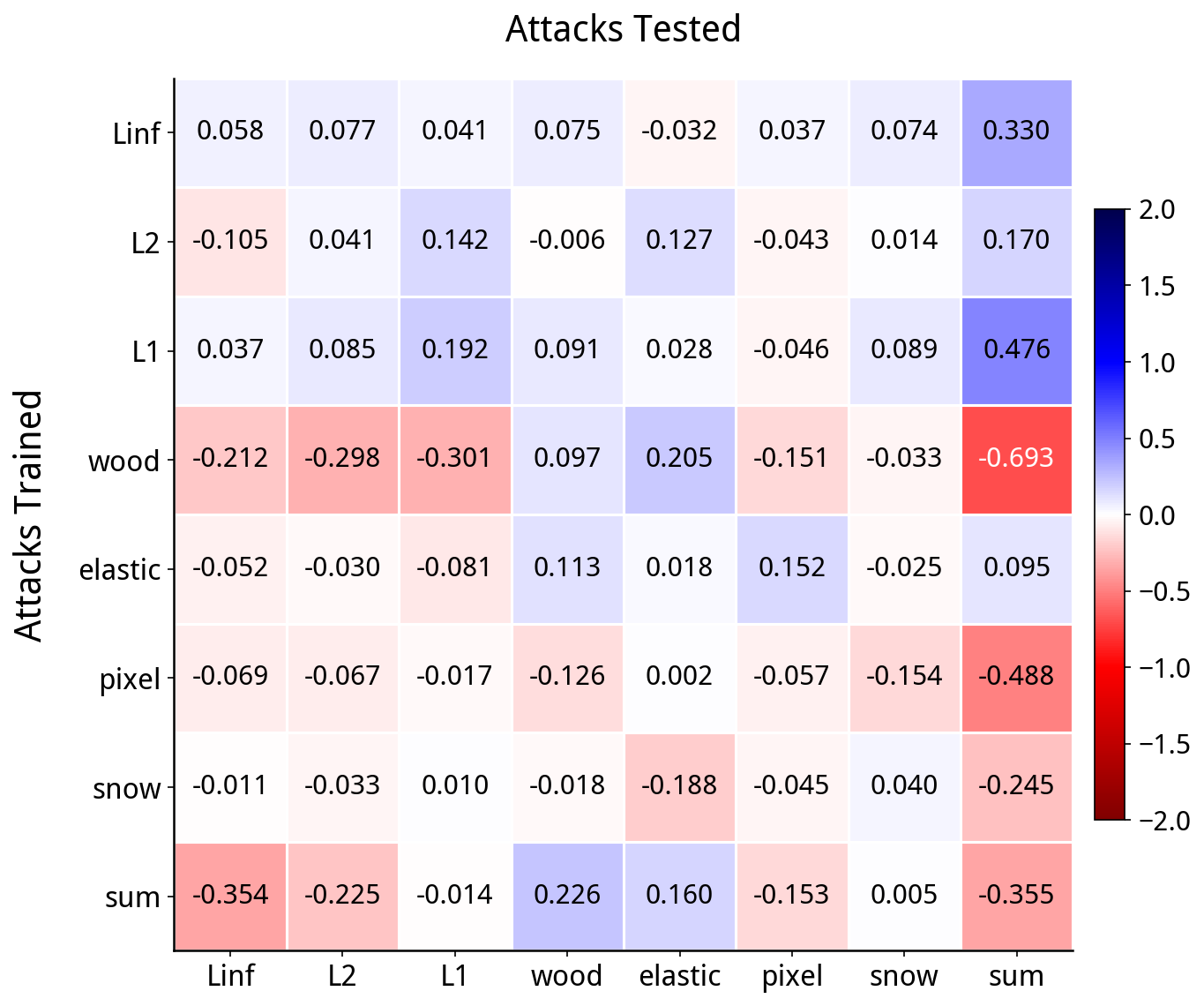}
    \vspace{-10pt}
    \caption{Trade-off matrix visualization. Each entry represents the change in robust accuracy of a specific attack type (shown along the top) after 3 epochs of sequential fine-tuning against designated attack types (shown on the left).}
    \vspace{-10pt}
    \label{fig:tradeoff}
    \Description{tradeoff}
\end{figure}

\noindent\textbf{Semantic-$\ell_p$ Conflicts}. As shown in Figure~\ref{fig:tradeoff}, semantic attacks often degrade $\ell_p$ robustness because most natural corruptions are incompatible with $\ell_p$ perturbation constraints.

\noindent\textbf{Transfer Asymmetry}. The tradeoff matrix shows no symmetric pattern with respect to the diagonal. For instance, AT against $\ell_p-$attack enhances robustness against most semantic attacks, while AT using most semantic attacks tends to impair $\ell_p$ robustness.

\noindent\textbf{Robustness Interference}. When summing all values in the trade-off matrix, it yields a negative total of \textbf{-0.355}. Furthermore, the marginal gain in robustness from pixel adversarial training itself is negative, indicating that the prior parameters may be far away from the relatively stable region for optimizing the robustness loss.

\subsection{Computational Analysis of Sequential Fine-Tuning Failure}

The results of the preliminary experiment, where the sum of robust accuracies decreased after fine-tuning against individual adversarial attacks, have already demonstrated that such a simple sequential strategy is infeasible for multi-robustness fine-tuning. It is therefore necessary to adopt alternative strategies to mitigate these \emph{cross-type trade-offs} and prevent \emph{catastrophic forgetting} of previously trained attacks during fine-tuning. Before that, however, we first explore the theoretical reasons behind the failure of sequential fine-tuning to better understand this trade-off.

\subsubsection{Theoretical Formulation}
Let $\mathcal{S}_p$ and $\mathcal{S}_q$ denote two distinct attack types, each characterized by its perturbation range $\Delta_i$ for $i\in\{p,q\}$. The robust risk for each attack $i\in\{p,q\}$ is defined as: 
\begin{equation} \mathcal{R}_i(\theta)=\mathbb{E}_{(x,y)\sim\mathcal{D}}\!\left[\max_{\delta\in\Delta_i}\mathcal{L}(f_\theta(x+\delta),y)\right],
\end{equation}
where $\mathcal{L}$ is the loss function, and $\mathcal{D}$ denotes the clean data distribution.
We then define the average robust risk as:

\begin{equation}
\mathcal{R}_{avg}(\theta)=\tfrac{1}{2}\big(\mathcal{R}_p(\theta)+\mathcal{R}_q(\theta)\big)
\end{equation}

\subsubsection{Main Results}
\begin{lemma}[Average robust risk bound]
\label{lemma}
The change in average robust risk $\Delta\mathcal{R}_{avg}$ can be bounded as:
\begin{equation}
\Delta\mathcal{R}_{avg} \lesssim -\|\nabla\mathcal{R}_{avg}(\theta_1)\| \|\Delta\theta\| \cos\psi + \tfrac{1}{2}\lambda_{\max}(H_{avg}) \|\Delta\theta\|^2
\end{equation}
where $\psi$ denotes the angle between $\nabla\mathcal{R}_{avg}(\theta_1)$ and $\nabla\mathcal{R}_{q}(\theta_1)$, $H_{avg}$ is the local Hessian of $\mathcal{R}_{avg}$, $\lambda_{\max}(H_{avg})$ denotes its largest eigenvalue of the Hessian at $\theta_1$.
\end{lemma}

Based on the preceding lemma, we can derive a safe parameter-drift threshold to prevent the average robust risk from increasing during fine-tuning
\begin{proposition}[Safe parameter-drift threshold]
\label{prop}
Let \\$\lambda_{\max}(H_{avg})$ denote the largest eigenvalue of the local Hessian of $\mathcal{R}_{avg}$ at $\theta_1$. A sufficient condition to prevent average robustness degradation after the sequential fine-tuning in the order of $S_p$ followed by $S_q$ is:
\begin{equation}
\label{eq:critical_drift}
\|\Delta\theta\| < \frac{2 \|\nabla\mathcal{R}_{avg}(\theta_1)\| \cos\psi}{\lambda_{\max}(H_{avg})}, \quad \psi < \frac{\pi}{2}
\end{equation}
\end{proposition}

This reveals that degradation arises either from weak gradient alignment ($\psi$ is relatively large, making $\cos\psi$ small), excessive drift $\Delta\theta$, or from large curvature $\lambda_{\max}(H_{avg})$. The detailed assumptions and proofs are provided in Appendix~B.

\subsubsection{Empirical Understanding}
Based on Proposition~\ref{prop}, we identify \emph{parameter drift} as the primary mechanism behind the failure of sequential fine-tuning. A sufficiently large displacement in parameter space can push the model out of regions robust to previous attack types, leading to \emph{catastrophic forgetting}. This parameter drift arises mainly from \emph{gradient conflicts} among different adversarial objectives, which fundamentally stem from distinct constraint structures of these attacks. Moreover, excessive optimization on a single direction further amplifies such drift. 

These observations motivate the design of our Calibrated Adversarial Sampling (CAS) method.
CAS employs high-frequency random sampling to counteract parameter drift and explicitly incorporates the cross-type trade-off among multiple robustness objectives.
This enables it to optimize parameter updates along safer directions and maintain balanced robustness across attacks.
Furthermore, inspired by multi-objective optimization theory, we introduce the cross-type trade-off structure into the reward function to guide updates along the \emph{Pareto frontier}, balancing exploration and exploitation, as detailed in Appendix~C.


\section{Methodology}
\label{sec:methodology}

We cast the problem of fine-tuning for multi-type robustness as a complex variant of the multi-armed bandit problem. By designing a dynamic and fair reward function and balancing exploration with exploitation, we transform the complex problem of jointly optimizing multiple robustness objectives into an interpretable and extensible framework, named Calibrated Adversarial Sampling (CAS), as illustrated in Figure~\ref{fig:pipeline}.

\subsection{Problem Formulation}
\label{subsec:problem}

Consider a classification model $f_\theta$ trained on data distribution $\mathcal{D}$, and a set of $M$ adversarial perturbation types $\mathcal{P} = \{ p_1, p_2, \dots, p_M \}$. Each $p_k$ corresponds to an adversarial attack type (e.g., $\ell_\infty$, blur) and has an importance weight $w_k$. At each training iteration $t$, we select a perturbation type $p_{a_t}$, generate adversarial examples under this type, and update model parameters $\theta_t \mapsto \theta_{t+1}$ accordingly. 

We view the choice of $p_{a_t}$ as pulling arm $a_t$ in a multi-armed bandit. The ``reward'' quantifies the overall improvement in multi-type robustness achieved by training on this perturbation type. However, unlike standard bandits, the underlying reward process here is: (a) \textbf{Non-stationary:} The robustness improvement achieved by each adversarial training step is not fixed, and can only be dynamically evaluated through the loss and robust accuracy on the corresponding adversarial examples; (b) \textbf{Partially observable:} Adversarial training against one perturbation can simultaneously increase or decrease robustness against other perturbations. Since the real-time impact against other perturbation types cannot be directly observed, these dynamics are thus better modeled as a partially observable Markov process.

Formally, let $L_m^{(t)}$ denote the adversarial loss under perturbation $p_m$ at iteration $t$, with the observed feedback given by the loss sequence $(L_{k_0}^{(0)}, L_{k_1}^{(1)}\ldots, L_{k_t}^{(t)})$. Since the evolution of multi-type robustness is unknown and non-stationary, designing a comprehensive and dynamic reward function is challenging yet crucial.

\subsection{Reward Design}
\label{subsec:reward}

At each iteration, the reward for selecting perturbation $p_v$ is composed of two parts:
\begin{equation}
    R_v = R_v^{\text{self}} + R_v^{\text{tradeoff}}
\end{equation}
where $R_v^{\text{self}}$ captures the marginal gain in robustness under $p_v$ itself, and $R_v^{\text{tradeoff}}$ measures its effect on other perturbation types.

\paragraph{Marginal Robustness Gain.}
To estimate the marginal effect of training against the $p_v$ attack, we examine the trajectory of $\log L_v$ over the most recent iterations in which $p_v$ is actually applied. Specifically, we record the indices of the latest at most $W$ such iterations up to step $t$, denoted by $\mathcal{T}_v^{(t)} = \{t_1, \dots, t_m\}$ with $m \le W$, and compute the slope of the linear regression line of the corresponding log-transformed losses with respect to the iteration indices. The negative slope is then used to represent the marginal gain:
\begin{equation}
R_v^{\text{self}} \;=\;
\begin{cases}
0, & m = 0, 1 \\[6pt]
- \left(
\frac{
\sum_{j=1}^{m} (t_j - \bar{t}) \big( \log L_v^{(t_j)} - \overline{\log L_v} \big)
}{
\sum_{j=1}^{m} (t_j - \bar{t})^2
}
\right) \cdot w_v, & m \ge 2
\end{cases}
\label{eq:reward_self}
\end{equation}
where $\bar{t} = \frac{1}{m} \sum_{j=1}^{m} t_j$ and 
$\overline{\log L_v} = \frac{1}{m} \sum_{j=1}^{m} \log L_v^{(t_j)}$.
Applying the logarithm promotes fairness across different robustness by normalizing loss magnitudes, which can vary significantly between attacks, thus reducing sensitivity to the perturbation strength $\epsilon_k$. 
The sliding window smooths noise while capturing recent trends, and a steeper negative slope indicates a faster robustness improvement.

\paragraph{Cross-Type Trade-off.}
Training on $p_v$ can influence losses under other perturbations $p_k$ ($k \neq v$), potentially yielding either improvements or degradation due to conflicting constraints. For each $p_k$, let $t_{\text{prev}}^k$ and $t_{\text{curr}}^k$ denote its two most recent occurrences (if they exist). We count how many times each perturbation is selected between these two steps, and denote by $n_{j,k}$ the number of times $p_j$ is applied in this interval. The trade-off reward is defined as:
\begin{equation}
R_v^{\text{tradeoff}} =
\sum^M_{\substack{
k=1,k \neq v \\
\exists t_{\text{prev}}^k ,\;
\exists t_{\text{curr}}^k
}}
w_k \cdot
\frac{n_{v,k}}{\sum_{j=1}^M n_{j,k}} \cdot
\log \frac{L_k^{(t_{\text{prev}}^k)}}{L_k^{(t_{\text{curr}}^k)}}.
\label{eq:reward_tradeoff}
\end{equation}
This formulation attributes each change in $L_k$ ($k\neq v$) to the $p_v$ perturbations applied in the interval proportionally, rewarding positive cross-type transfer and penalizing degradation, while the logarithmic ratio ensures scale invariance.

\paragraph{Hybrid Loss.}
To balance clean accuracy and robustness, we adopt a TRADES-style~\cite{zhang2019theoretically} hybrid loss:
\begin{equation}
\mathcal{L}_{\mathrm{total}} 
= \beta \, \mathcal{L}\big( f_\theta(x^{\mathrm{adv}}), y \big) 
+ (1-\beta) \, \mathcal{L}\big( f_\theta(x), y \big)
\label{eq:trades}
\end{equation}
where $x^{\mathrm{adv}}$ is the adversarial example generated under the selected attack, and $\beta$ is a weighting hyperparameter. This formulation prevents catastrophic degradation of clean accuracy during AT.

\subsection{Balancing Exploration and Exploitation}
\label{subsec:ucb}

To effectively fine-tune across multiple perturbation types,
it is essential to balance \textit{exploitation}, focusing on high-reward perturbations with \textit{exploration} of under-sampled ones to ensure accurate update of rewards for all perturbation types. To achieve this, we utilize the Upper Confidence Bound (UCB) algorithm from multi-armed bandit theory in our sampling strategy.

Let $N_v$ denote the number of times that perturbation $p_v$ has been selected (initialized as $N_v=1$). The final selection score is:
\begin{equation}
\tilde{R}_v \;=\; \exp(\alpha \cdot R_v) \;+\; \sqrt{\frac{2\log (\sum_{j=1}^M N_j)}{N_v}},
\label{eq:ucb}
\end{equation}
where $\alpha$ is a hyperparameter that controls the scaling of rewards.
The exponential term ensures non-negativity and enables smooth weighting under the softmax sampling scheme, while the second term encourages exploration by increasing the score for rarely-sampled perturbations.
As training progresses and $N$ increases, 
the exploration term gradually diminishes in magnitude relative to $N_v$, naturally shifting the selection strategy from exploration toward exploitation. 
Sampling probabilities are then defined as:
\begin{equation}
\pi_v \;=\; \frac{\tilde{R}_v \cdot w_v}{\sum_{j=1}^M \tilde{R}_j \cdot w_j},
\label{eq:sampling}
\end{equation}
where $w_v$ denotes an optional user-specified importance weight. 
At each iteration, perturbation types are chosen according to $\pi_v$ . This strategy allows CAS to adaptively allocate adversarial samples towards those with higher overall robustness gains while ensuring sufficient exploration to maintain fairer coverage and mitigate catastrophic forgetting of previously trained types.

\begin{algorithm}
\caption{Calibrated Adversarial Sampling (CAS)}
\label{alg:cas}
\begin{algorithmic}[1]
\REQUIRE Pre-trained model parameters $\theta$, perturbations $\mathcal{P} = \{p_1, \dots, p_M\}$ with weights $w_1, \dots, w_M$, sliding window size $W$, hyperparameter $\alpha$ and $\beta$, total iterations $T$
\ENSURE Fine-tuned model parameters $\theta_T$
\STATE Initialize $N_v \gets 0$ for all $v \in [M]$, loss history buffers
\FOR{$t = 1$ \TO $T$}
    \FOR{$v = 1$ \TO $M$}
        \STATE Compute $R_v^{\text{self}}$ using loss buffers (Eq.~\ref{eq:reward_self})
        \STATE Compute $R_v^{\text{tradeoff}}$ using cross-type records (Eq.~\ref{eq:reward_tradeoff})
        \STATE $R_v \gets R_v^{\text{self}} + R_v^{\text{tradeoff}}$ 
        \STATE $\tilde{R}_v \gets \exp(\alpha \cdot R_v) + \sqrt{\frac{2\log (\sum_{j=1}^M N_j)}{N_v}}$ 
    \ENDFOR
    \STATE Sample $a_t \sim \pi$ where $\pi_v = \frac{\tilde{R}_v \cdot w_v}{\sum_j \tilde{R}_j \cdot w_j}$ 
    \STATE $N_{a_t} \gets N_{a_t} + 1$
    \STATE Generate $x^{\mathrm{adv}}$ under perturbation $p_{a_t}$
    \STATE $\mathcal{L}_{\mathrm{total}} \gets \beta \mathcal{L}(f_\theta(x^{\mathrm{adv}}), y) + (1-\beta)\mathcal{L}(f_\theta(x), y)$
    \STATE Update $\theta_t$ via $\nabla_\theta \mathcal{L}_{\mathrm{total}}$
    \STATE Update $L_{a_t}^t$ and cross-type records
\ENDFOR
\end{algorithmic}
\end{algorithm}

\section{Experiment}
\begin{table*}[t]
\centering
\caption{Overall comparison of our CAS method with baselines on CIFAR-10, CIFAR-100 and SVHN datasets. The “Original” model is obtained via $\ell_p$ pre-training on each dataset, and “Order” denotes sequential fine-tuning by cycling through attack types in a fixed order. \textbf{avg. $\ell_p$} denotes the average robust accuracy over $\ell_\infty$, $\ell_2$, and $\ell_1$, while \textbf{avg. Corruption} denotes the average robust accuracy against snow, fog, and blur attacks. All results are reported as mean $\pm$ standard deviation over five independent runs. The average time cost of each method across the three datasets is reported as follows.}
\begin{tabular}{c|c|cccccc|c}
\toprule
\textbf{Dataset} & \textbf{Accuracy} & \textbf{Original} & \textbf{E-AT} & \textbf{SAT} & \textbf{Order} & \textbf{AVG} & \textbf{PCGrad} & \textbf{CAS (ours)} \\
\midrule
\multirow{4}{*}{CIFAR-10}&Clean          & 82.90$\pm$0.00 & 83.56$\pm$0.31  & 84.62$\pm$0.42 & 82.92$\pm$0.17 & 84.56$\pm$0.24 & 84.44$\pm$0.21 & \textbf{85.26$\pm$0.33} \\
&avg. Robust           & 33.51$\pm$0.00 & 51.50$\pm$0.33 & 51.41$\pm$0.40 & 51.23$\pm$0.15 & 51.46$\pm$0.23 & \textbf{52.07$\pm$0.15} & 51.79$\pm$0.35 \\
&avg. $\ell_p$  & 31.13$\pm$0.00 & 52.86$\pm$0.60 & 53.81$\pm$0.54 & \textbf{53.95$\pm$0.18} & 53.47$\pm$0.31 & 53.78$\pm$0.32 & 52.91$\pm$0.54 \\
&avg. Semantic      & 35.88$\pm$0.00 & 50.14$\pm$0.31 & 49.00$\pm$0.27 & 48.51$\pm$0.13 & 49.46$\pm$0.18 & 50.36$\pm$0.21 & \textbf{50.67$\pm$0.22} \\
\midrule
\multirow{4}{*}{CIFAR-100}&Clean          & 48.00$\pm$0.00 & 57.40$\pm$0.39  & 57.92$\pm$0.50 & 56.94$\pm$0.31 & 57.78$\pm$0.30 & 57.86$\pm$0.20 & \textbf{58.54$\pm$0.48} \\
&avg. Robust           & 16.61$\pm$0.00 & 27.62$\pm$0.48 & 27.53$\pm$0.46 & 27.12$\pm$0.30 & 27.43$\pm$0.32 & \textbf{28.14$\pm$0.16} & 27.99$\pm$0.37 \\
&avg. $\ell_p$  & 19.03$\pm$0.00 & 28.60$\pm$0.67 & 28.69$\pm$0.50 & 28.75$\pm$0.29 & 28.56$\pm$0.37 &\textbf{29.87$\pm$0.23} & 29.23$\pm$0.46 \\
&avg. Semantic      & 14.19$\pm$0.00 & 26.65$\pm$0.43 & 26.36$\pm$0.46 & 25.48$\pm$0.33 & 26.29$\pm$0.32 & 26.46$\pm$0.18 & \textbf{26.75$\pm$0.39} \\
\midrule
\multirow{4}{*}{SVHN}&Clean          & 80.40$\pm$0.00 & 92.06$\pm$0.46  & 93.32$\pm$0.66 & 92.82$\pm$0.35 & 93.02$\pm$0.39 & 93.22$\pm$0.31 &\textbf{93.86$\pm$0.54} \\
&avg. Robust           & 40.20$\pm$0.00 & 56.32$\pm$0.71 & 55.91$\pm$0.73 & 53.03$\pm$0.49 & 56.33$\pm$0.41 & 57.58$\pm$0.36 &\textbf{57.76$\pm$0.67} \\
&avg. Corruption  & 41.67$\pm$0.00 & 62.86$\pm$0.67 & 63.62$\pm$0.52 & 60.63$\pm$0.66 & 63.51$\pm$0.44 & \textbf{64.31$\pm$0.35} &63.81$\pm$0.48 \\
&avg. Others      & 38.74$\pm$0.00 & 49.78$\pm$1.05 & 48.21$\pm$1.05 & 45.43$\pm$0.75 & 49.15$\pm$0.58 & 50.84$\pm$0.55 & \textbf{51.71$\pm$0.96} \\
\midrule
 - & \textbf{Time (s)}& 0.0 & 2390.7 & 2386.6 & 2375.4 & 47241.1 & 47824.4 &2408.0\\
\bottomrule
\end{tabular}
\label{main experiment}
\end{table*}
In this section, we demonstrate the effectiveness of our proposed CAS framework in improving overall robustness and addressing the multi-robust trade-off.
\subsection{Experimental Setup}
We conduct our experiments on the benchmark datasets CIFAR-10, CIFAR-100~\cite{krizhevsky2009learning} and SVHN~\cite{netzer2011reading} using pre-trained PreActResNet-18~\cite{he2016identity} models provided by~\cite{croce2022adversarial} and~\cite{rice2020overfitting}. During fine-tuning, we consider the 21 different adversarial attacks  presented in Section~\ref{subsec: Unforeseen Adversarial Attacks}. 

\noindent \textbf{Baselines}. We compare our CAS method with two representative high-frequency random fine-tuning baselines: SAT~\cite{madaan2021learning} and E-AT~\cite{croce2022adversarial}. To examine whether the order of adversarial attacks affects performance, we also include a simple sequential fine-tuning scheme of our own design, which cycles through attack types in a fixed, weight-based order. Additionally, we add the well-known AVG method~\cite{tramer2019adversarial} and PCGrad~\cite{yu2020gradient} for the comparison of multi-task optimization methods. Detailed descriptions of these methods are provided in Section~\ref{sec:related}.

\noindent \textbf{Training Settings}. Following the common practice of AT~\cite{pang2020bag,wang2022self,wei2023cfa}, we fine-tune a pretrained PreActResNet-18 model using SGD with momentum 0.9, weight decay $5 \times 10^{-4}$, and initial learning rate 0.1 for 10 epochs. The hyperparameters are set to $\alpha = 10$ and $\beta = \frac89$. On CIFAR-10 the pretrained model has been adversarially trained under the $\ell_\infty$ norm, while on CIFAR-100 and SVHN, the models have been adversarially trained under $\ell_2$. We limit the number of fine-tuning epochs as prior work~\cite{croce2022adversarial} has shown that quick fine-tuning can significantly improve multi-robustness, which is also corroborated by our own ablation study. 

All $\ell_p$ attacks are conducted using default perturbation margin $\epsilon_\infty= \frac{8}{255}$, $\epsilon_2=0.5$, and $\epsilon_1=12$. For semantic attacks, we compute a calibrated margin to ensure that most perturbations yield robust accuracies between 20\% and 60\% on the pre-trained CIFAR-10 model:
$\epsilon_{k}=(\lambda_k+acc_{adv}[k]) \times \epsilon$,
where $\epsilon$ is the original perturbation margin and $\lambda_k$ is a hyperparameter controlling the calibration. We repeat this process until $acc_{adv}[k]$ falls within the desired range.

\noindent \textbf{Evaluation Metrics}. We report the mean and standard deviation of both clean and robust accuracy over five independent runs. Robust accuracy is evaluated by AutoAttack~\cite{croce2020reliable}, a widely-used robustness evaluation benchmark.

\noindent \textbf{Real-World Scenario Simulation}. On CIFAR-10 and CIFAR-100, we assign weight 6 to each $\ell_p$ attack, and weight 1 to the 18 semantic attacks. This emulates a scenario where attackers can precisely manipulate inputs by adding norm-constrained perturbations. For the street-view SVHN dataset, we assign weight 6 to common corruptions (snow, blur, and fog) and weight 1 to the remaining 18 attacks, simulating an autonomous driving context where robustness to natural corruptions is critical. The average robust accuracy is the weighted sum using these scenario-specific weights.

\subsection{Main Results}

\noindent\textbf{Superior Holistic Accuracy.} As shown in Table~\ref{main experiment}, our CAS method achieves the highest clean accuracy and the first or second highest average robust accuracy on all datasets. It consistently delivers strong performance on CIFAR-10, CIFAR-100, and SVHN, while maintaining balanced robustness between high-weighted and other perturbation types. Notably, on SVHN, CAS surpasses all baselines in avg. Robust, demonstrating its potential for real-world scenarios such as autonomous driving. These improvements mainly stem from our reward design that explicitly accounts for robustness trade-offs, together with the UCB-based exploration–exploitation mechanism. The effectiveness of these components is further validated through ablation studies in Appendix~E.

\noindent\textbf{Analysis of Baselines.} Each baseline exhibits distinct characteristics. E-AT adaptively assigns higher selection probability to attacks with lower robust accuracy, enhancing fairness across robustness types; however, this tendency to favor stronger attacks often reduces clean accuracy. The random fine-tuning method SAT and the multi-task learning approach AVG yield similar performance, as analyzed from a Pareto frontier perspective in Appendix~C. PCGrad achieves high overall robust accuracy, but is excessively time-consuming. The Order method exhibits lower variance than random methods, but also lower accuracy, suggesting that a fixed order of attack selection may induce overfitting to the sampling sequence.

\noindent\textbf{Time Cost Analysis.} As reported in Table~\ref{main experiment}, our CAS method improves final accuracy with a time cost comparable to other baseline methods while outperforming the AVG and PCGrad methods. The time cost of AVG and PCGrad grows linearly with the number of adversarial attack types, revealing an inherent limitation of multi-task learning methods in scalability. In contrast, our CAS method significantly reduces time costs by using a sliding window to dynamically record the marginal gains and interaction effects of different adversarial attacks in real time.

\subsection{Ablation and Sensitivity Analysis}
\begin{figure}
    \centering
    \includegraphics[width=1.0\linewidth]{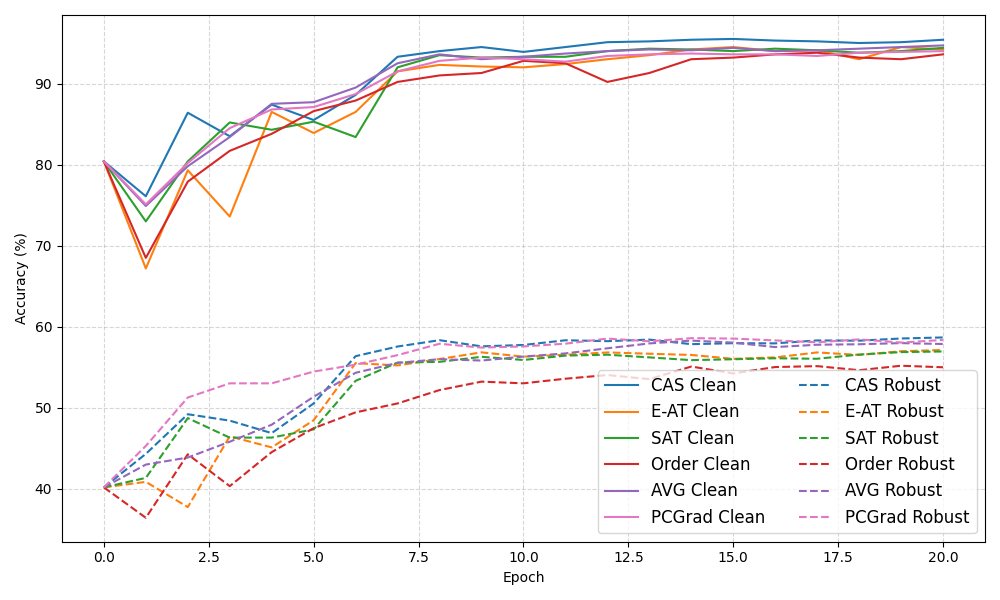}
    \caption{Ablation study examining the effect of training epochs using SVHN. Robust accuracy is the weighted average.}
    \label{fig:epoch}
    \Description{epoch}
\end{figure}
\noindent \textbf{Epochs.} The analysis results on training epochs (Figure~\ref{fig:epoch}) validate our choice of 10 epochs for the main experiments. During initial training (epochs 0-6), all methods exhibit significant instability in accuracy.
This aligns with prior findings that, in the early stages of adversarial training, the loss surface has not yet been smoothed, and optimization conflicts occur while the model learns robust features~\cite{zhang2019theoretically,liu2020loss,zhang2022revisiting}. Beyond epoch 10, clean accuracy stabilizes and robust accuracy shows only marginal fluctuations with diminishing returns. Our CAS method demonstrates stable superiority in this regime, maintaining robustness advantages without compromising clean performance. 
Therefore, choosing an epoch of 10 in our experiments achieves a good balance between performance and efficiency, avoiding unnecessary computational and overfitting risk while achieving strong multi-robustness.

\noindent \textbf{Number of Considered Perturbations.}
\begin{figure}
    \centering
    \includegraphics[width=1.0\linewidth]{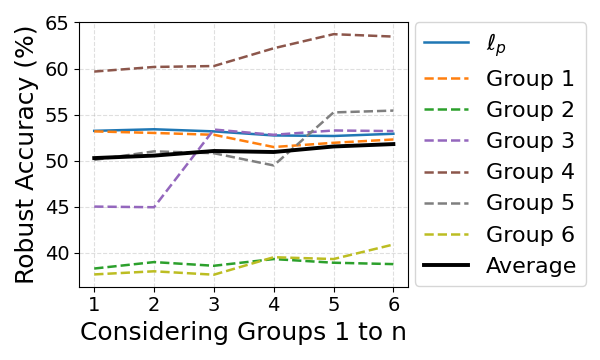}
    \caption{Ablation study on the number of considered perturbations using CIFAR-10. The 18 semantic attacks, (Wood, Elastic, Pixel), (Snow, Gabor, JPEG), (Glitch, Kaleidoscope, Blur), (Edge, Fog, Texture), (Prison, Whirlpool, Polkadot), (Klotski, Hsv and PPGD) are divided into Group 1 to 6 in order from left to right. We consider the attacks from Group 1 to Group n one by one and assign a weight of 1 to each attack, while each $\ell_p$ ($p=1,2,\infty$) attack is always included and assigned a weight of 6. We report the weighted average robust accuracy within each group and across all attacks.}
    \label{fig:number}
    \Description{number}
\end{figure}
As shown in Figure~\ref{fig:number}, incorporating more perturbation types during fine-tuning steadily improves the overall average robust accuracy, while the $\ell_p$ robustness shows a slight but stable decline. The robustness trends of different semantic attack groups also vary: for instance, the robustness of Group 3 and Group 5 show clear improvement when they are included in fine-tuning, whereas Group 1 and Group 5 experience minor degradation when Group 4 is introduced, indicating a negative trade-off caused by that group. In contrast, Group 2 appears insensitive to the inclusion of other attacks. 
Overall, expanding the range of perturbation types in CAS fine-tuning yields a  stable improvement in average robustness, with only limited negative effects on specific robustness types. These observations suggest that, in practice, introducing a broader variety of adversarial attacks during fine-tuning enhances general robustness.

\noindent \textbf{Bias between Exploration and Exploitation.} According to Eq.~\ref{eq:ucb}, a larger $\alpha$ makes the CAS method more inclined toward exploitation, while a smaller $\alpha$ encourages exploration. As shown in Table~\ref{tab: alpha}, gradually increasing $\alpha$ from 1 to 10 improves both clean accuracy and robustness. This is because a moderate increase in $\alpha$ enables CAS to better utilize reward information for optimal sampling allocation, thereby mitigating negative trade-offs. However, since $\alpha$ affects the reward exponentially in Eq.~\ref{eq:ucb}, excessively large $\alpha$ values cause CAS to overemphasize high-reward adversarial perturbations, leading to degradation in multi-robustness and increasing the risk of catastrophic forgetting. Empirically, setting $\alpha$ between 10 and 20 achieves a good balance between exploration and exploitation.

\noindent\textbf{Additional analysis}. We provide
more comprehensive ablation studies on more datasets and each key component in Appendix~E, where the results are consistent with claims in this section.


\begin{table}[t]
\centering
\caption{Sensitivity analysis of the hyperparameter $\alpha$, which controls the balance between exploration and exploitation in CAS, conducted on the CIFAR-10 dataset. Results are averaged over five independent runs.}
\begin{tabular}{c|cc|c|cc}
\toprule
\textbf{$\alpha$} & \textbf{1} & \textbf{5} & \textbf{10} & \textbf{20} & \textbf{50} \\
\midrule
Clean         & 84.32 & 84.68 & 85.26 & \textbf{85.42} & 85.04 \\
avg. Robust   & 51.68 & 51.50 & \textbf{51.79} & 51.55 & 50.83 \\
avg. $\ell_p$ & 53.10 & \textbf{53.27} & 52.91 & 52.83 & 50.60 \\
avg. Semantic & 50.26 & 49.73 & 50.67 & 50.26 & \textbf{51.05} \\
\bottomrule
\end{tabular}
\label{tab: alpha}
\end{table}
\section{Convergence Analysis of Randomized Multi-Attack Training}

\subsection{Problem Setup and Assumptions}
In this section, we establish the theoretical foundations for the convergence of the randomized multi-attack training methods including CAS under standard optimization assumptions.

Consider a set of $K$ adversarial attack types $\mathcal{A} = \{a_1, \dots, a_K\}$. For each attack $i$, let $\mathcal{L}_i(\theta)$ be its loss function. In the CAS method, at each iteration $t$, we select an attack $i_t$ following a dynamic probability distribution $\pi_t = \{p_{1,t}, \dots, p_{K,t}\}$. The update rule is:
\begin{equation}
    \theta_{t+1} = \theta_t - \eta_t \nabla \mathcal{L}_{i_t}(\theta_t)
\end{equation}
To ensure the algorithm optimizes a meaningful objective, we assume the time-averaged sampling distribution converges to a stationary distribution $\pi^* = \{p_1^*, \dots, p_K^*\}$: $\forall i \in \{1, \dots, K\} , $
$\lim_{T \to \infty} \frac{1}{T} \sum_{t=1}^T p_{i,t} = p_i^*.$
This implies the sequence of updates asymptotically minimizes the aggregate risk $\mathcal{L}^*(\theta) = \sum_{i=1}^K p_i^* \mathcal{L}_i(\theta)$. Let $\theta^*$ be the unique minimizer of $\mathcal{L}^*(\theta)$.

\begin{assumption}[$\mu-$Strong Convexity]
    \label{assum:strong_convex}
    There exists $\mu > 0$ such that for any distribution $\pi_t$ in the sequence and for all $\theta$, the instantaneous expected gradient $\nabla \mathcal{L}_t(\theta) = \sum_{i=1}^K p_{i,t} \nabla \mathcal{L}_i(\theta)$ satisfies:
    \begin{equation}
        \langle \nabla \mathcal{L}_t(\theta), \theta - \theta^* \rangle \geq \mu \|\theta - \theta^*\|^2
    \end{equation}
\end{assumption}

\begin{assumption}[Consistent Variance Bound]
    \label{assum:variance}
    The second moment of the stochastic gradient is uniformly bounded for all $\pi_t$:
    \begin{equation}
        \mathbb{E}[\|\nabla \mathcal{L}_{i_t}(\theta_t)\|^2 \mid \mathcal{F}_t] = \sum_{i=1}^K p_{i,t} \|\nabla \mathcal{L}_i(\theta_t)\|^2 \leq G^2
    \end{equation}
\end{assumption}

\subsection{Proof of Convergence}

\begin{proof}
Let $V_t = \|\theta_t - \theta^*\|^2$ and $e_t = \mathbb{E}[V_t]$. We set the learning rate decay to $\eta_t = C/t$.

Starting from the update rule $\theta_{t+1} = \theta_t - \eta_t g_t$:
\begin{equation}
\begin{aligned}
    \|\theta_{t+1} - \theta^*\|^2 &= \|(\theta_t - \theta^*) - \eta_t g_t\|^2 \\
    &= \|\theta_t - \theta^*\|^2 - 2\eta_t \langle \theta_t - \theta^*, g_t \rangle + \eta_t^2 \|g_t\|^2
\end{aligned}
\end{equation}
Taking the conditional expectation $\mathbb{E}[\cdot \mid \mathcal{F}_t]$ and using the unbiasedness $\mathbb{E}[g_t \mid \mathcal{F}_t] = \nabla \mathcal{L}_t(\theta_t)$:
\begin{equation}
    \mathbb{E}[V_{t+1} \mid \mathcal{F}_t] = V_t - 2\eta_t \langle \theta_t - \theta^*, \nabla \mathcal{L}_t(\theta_t) \rangle + \eta_t^2 \mathbb{E}[\|g_t\|^2 \mid \mathcal{F}_t]
\end{equation}
Applying Assumptions \ref{assum:strong_convex} and \ref{assum:variance}:
\begin{equation}
    \mathbb{E}[V_{t+1} \mid \mathcal{F}_t] \leq V_t - 2\eta_t \mu V_t + \eta_t^2 G^2 = (1 - 2\mu \eta_t) V_t + \eta_t^2 G^2
\end{equation}
Using $e_{t+1} = \mathbb{E}[\mathbb{E}[V_{t+1} \mid \mathcal{F}_t]]$ and $\eta_t = C/t$:
\begin{equation}
    \label{eq:master_recurrence}
    e_{t+1} \leq \left( 1 - \frac{2\mu C}{t} \right) e_t + \frac{C^2 G^2}{t^2}
\end{equation}

We prove $e_t \leq \frac{K_{bound}}{t}$ by induction. Let $K_{bound} = \max \left\{ e_1, \frac{C^2 G^2}{2\mu C - 1} \right\}$, which implies $K_{bound} \geq \frac{C^2 G^2}{2\mu C - 1}$. This choice is valid if $C > 1/2\mu$.

Base case $t=1$: $e_1 \leq K_{bound}$ is true by definition. 
Assume $e_t \leq \frac{K_{bound}}{t}$ holds. Then for $t+1$:
\begin{equation}
\begin{aligned}
    e_{t+1} &\leq \left( 1 - \frac{2\mu C}{t} \right) \frac{K_{bound}}{t} + \frac{C^2 G^2}{t^2} \\
    &= \frac{K_{bound}}{t} - \frac{2\mu C K_{bound} - C^2 G^2}{t^2}
\end{aligned}
\end{equation}
By our choice of $K_{bound}$, $-(2\mu C K_{bound} - C^2 G^2) \leq -K_{bound}$:
\begin{equation}
    e_{t+1} \leq \frac{K_{bound}}{t} - \frac{K_{bound}}{t^2} = \frac{K_{bound}(t-1)}{t^2}
\end{equation}
Since $(t-1)(t+1) < t^2$, we have $\frac{t-1}{t^2} < \frac{1}{t+1}$. Therefore:
\begin{equation}
    e_{t+1} \leq \frac{K_{bound}}{t+1}
\end{equation}
This completes the induction, showing $e_t = \mathcal{O}(1/t)$ under such parameter settings where $\eta_t = C/t$, $C > 1/2\mu$ and $K_{bound} = \max \left\{ e_1, \frac{C^2 G^2}{2\mu C - 1} \right\}$.
\end{proof}

\subsection{Discussion and Limitations}



The convergence analysis demonstrates that, the CAS optimization process converges to a local optimum with an $O(1/t)$ rate by adopting a decaying learning rate $\eta_t=C/t$ under appropriate parameter settings. Although our experiments use a fixed learning rate of $0.1$, this follows common practice for fine-tuning pre-trained models rather than training from scratch, where only a small number of epochs are required. The theoretical learning-rate schedule provides an asymptotic convergence guarantee, while the fixed learning rate is to achieve a more pronounced and faster fine-tuning effect. 

Furthermore, the proof remains valid with a time-varying sampling distribution $\pi_t$, as long as the sampled gradients satisfy the bounded variance condition $\mathbb{E}[\|\nabla\mathcal{L}_{i_t}(\theta_t)\|^2|\mathcal{F}_t]\leq G^2$ and the average sampling probability converges. We acknowledge that the strong convexity assumption is idealized for practical deep learning landscapes. Nevertheless, the Accuracy-vs-Epoch curves (Figure~\ref{fig:epoch}) suggest that adaptive attack sampling combined with appropriate optimization schedules may provide an implicit regularization effect in non-convex settings, motivating future extensions under weaker assumptions such as $L$-smoothness.
\section{Related Work}
\label{sec:related}

\subsection{Adversarial Examples}
Adversarial examples are firstly discovered as deceptive samples crafted by applying subtle (often imperceptible) perturbations to clean inputs~\cite{szegedy2013intriguing}, which can mislead DNNs into making erroneous predictions. In practice, such perturbations are typically constrained within specific norm balls (\textit{e.g.}, $\ell_p$-norm constraints), or actual perturbations under deformation, compression (\textit{e.g.}, JPEG), or natural disturbances (\textit{e.g.}, fog), as well as noise generated by algorithmically simulating such perturbations~\cite{hsiung2023towards}. To deal with this threat, \textit{adversarial training} (\textbf{AT}) has emerged as the primary defense paradigm~\cite{wong2020fast,DBLP:conf/ijcai/BaiL0WW21,shafahi2019adversarial}, which enhances model robustness by explicitly injecting adversarial examples during training.


\subsection{Multi-Robustness of DNNs}
Although adversarial training with a single attack type can effectively improve robustness against that specific type of attack, real-world scenarios often require maintaining robustness against multiple kinds of adversarial perturbations. For example, a DNN used for medical image analysis may encounter deformations, flow-like distortions, and bubble artifacts~\cite{repetto2026evaluating}. For ethical and safety considerations, it is essential to ensure that the used DNN model maintains a certain level of robustness against all these perturbations. Due to the large variety of attacks and the inherent trade-offs between different types of robustness~\cite{kang2019testing}, models trained from scratch often struggle to achieve both comprehensive robustness and fairness across attacks. In practice, fine-tuning is commonly used to efficiently obtain models that meet multiple robustness requirements. Stochastic adversarial training (SAT)~\cite{madaan2021learning} addresses this by randomly selecting attack types with fixed probabilities. E-AT~\cite{croce2022adversarial} improves upon this by dynamically adjusting the selection probabilities based on robust accuracy, focusing more on relatively weaker robustness types. Multi-objective optimization methods, such as AVG~\cite{tramer2019adversarial} and PCGrad~\cite{yu2020gradient}, compute the samples and gradients of all adversarial attacks at each step, and process these gradients (e.g., weighted averaging, projection) to determine the optimization direction. But when the number of attack types becomes large, this approach incurs significant computational overhead.

\subsection{Multi-Armed Bandit Problem}
The multi-armed bandit (MAB) problem is a classical framework for sequential decision-making under uncertainty, where an agent repeatedly selects from a set of actions (``arms’’) and receives stochastic rewards~\cite{robbins1952some}. The core challenge is to balance \emph{exploration} (gathering information about uncertain arms) and \emph{exploitation} (selecting arms believed to yield the highest rewards)~\cite{lattimore2020bandit}. Among the most widely used strategies for achieving this balance are the Upper Confidence Bound (UCB) algorithm~\cite{auer2002using} and Thompson Sampling~\cite{thompson1933likelihood,russo2018tutorial}. UCB methods construct optimistic confidence intervals around each arm’s estimated reward and select the arm with the highest upper bound. A detailed introduction and theoretical derivation of these two methods can be found in Appendix~A.
\section{Conclusion}
In this paper, we introduced Calibrated Adversarial Sampling (CAS), a novel fine-tuning framework designed to enhance the robustness of DNNs against a broad range of adversarial attacks. Through theoretical analysis and empirical evaluation on robustness conflict, we investigated the inherent cross-type trade-offs among multiple robustness dimensions and developed a dynamic reward mechanism to address them. Our framework also bridges the multi-armed bandit framework and multi-attack adversarial training, enabling stable and efficient fine-tuning for applications. Extensive experiments demonstrate that CAS achieves superior accuracies and computational efficiency across diverse attack scenarios, establishing a practical paradigm for real-world robust generalization of DNNs.

\section*{Acknowledgement}
This research was sponsored by National Natural Science Foundation of China (Grant No. 92582102, 62572013, 62172019) and Beijing Natural Science Foundation, China (Grant No. QY24035, QY26045).

\bibliographystyle{ACM-Reference-Format}
\bibliography{ref}

\newpage
\clearpage
\appendix
\setcounter{lemma}{0}
\setcounter{proposition}{0}

\section{Theoretical Foundations of MAB Algorithms}
\label{app:mab}

This section provides a concise theoretical formulation and comparative analysis of two foundational algorithms for the stochastic multi-armed bandit (MAB) problem: the Upper Confidence Bound (UCB) algorithm and Thompson Sampling.

\subsection{Problem Formulation}

We consider a stochastic multi-armed bandit with $K$ arms. When an arm $a \in \{1, \dots, K\}$ is pulled at time $t$, it yields a reward $r_t$ drawn from a fixed but unknown distribution with mean $\mu_a$. The objective of a policy $\pi$ is to maximize the cumulative reward over a horizon $T$, or equivalently, to minimize the cumulative pseudo-regret $R_T$:
\begin{equation}
R_T = T \mu^* - \mathbb{E} \left[ \sum_{t=1}^{T} \mu_{a_t} \right]
\end{equation}
where $\mu^* = \max_{a} \mu_a$ is the reward mean of the optimal arm.

\subsection{Upper Confidence Bound (UCB) Algorithm}

The UCB algorithm is a deterministic method grounded in the optimism in the face of uncertainty principle. The core idea is to maintain an optimistic estimate of each arm's reward potential and select the arm with the highest upper confidence bound.

Let $\hat{\mu}_a(t)$ be the empirical mean reward of arm $a$ after it has been pulled $n_a(t)$ times by round $t$. The Chernoff-Hoeffding inequality states that for bounded random variables $X_i \in [0,1]$:
\begin{equation}
\mathbb{P}(|\hat{\mu} - \mu| \geq \varepsilon) \leq 2 \exp(-2n\varepsilon^2)
\end{equation}
Applying this to arm $a$ after $n_a(t)$ pulls:
\begin{equation}
\mathbb{P}(|\hat{\mu}_a(t) - \mu_a| \geq \varepsilon) \leq 2 \exp(-2n_a(t)\varepsilon^2)
\end{equation}
We want to find $U_a(t)$ such that:
\begin{equation}
\mathbb{P}(\mu_a > \hat{\mu}_a(t) + U_a(t)) \leq \delta
\end{equation}
Setting the right-hand side equal to $\delta$:
\begin{equation}
\exp(-2n_a(t)U_a(t)^2) = \delta
\end{equation}
Solving for $U_a(t)$:
\begin{equation}
-2n_a(t)U_a(t)^2 = \ln \delta 
U_a(t)^2 = \frac{-\ln \delta}{2n_a(t)} 
U_a(t) = \sqrt{\frac{-\ln \delta}{2n_a(t)}}
\end{equation}
Choosing $\delta = t^{-2}$ to ensure the sum of failure probabilities converges:
\begin{equation}
U_a(t) = \sqrt{\frac{-\ln(t^{-2})}{2n_a(t)}} = \sqrt{\frac{\ln t}{n_a(t)}}
\end{equation}
The UCB algorithm uses a slightly more conservative bound:
\begin{equation}
\text{UCB}_a(t) = \hat{\mu}_a(t) + \sqrt{\frac{2\ln t}{n_a(t)}}
\end{equation}
This ensures that the probability of overestimating any arm's true mean value decreases quickly over time.

This leads to the UCB index for each arm $a$:
\begin{equation}
\text{UCB}_a(t) = \hat{\mu}_a(t) + \sqrt{\frac{2 \ln t}{n_a(t)}}
\end{equation}
The algorithm selects the arm $a_t$ at time $t$ as:
\begin{equation}
a_t = \arg\max_{a} \left[ \hat{\mu}_a(t) + \sqrt{\frac{2 \ln t}{n_a(t)}} \right]
\end{equation}

The first term, $\hat{\mu}_a(t)$, encourages exploitation favoring arms with high empirical rewards. The second term, the confidence bonus, promotes exploration of less frequently pulled arms.

\begin{algorithm}[h]
\caption{UCB Algorithm}
\label{alg:ucb}
\begin{algorithmic}[1]
\STATE Initialize: Pull each arm once
\FOR{$t = K+1, K+2, \dots, T$}
    \FOR{each arm $a = 1, \dots, K$}
        \STATE Compute $\text{UCB}_a(t) = \hat{\mu}_a(t) + \sqrt{\frac{2 \ln t}{n_a(t)}}$
    \ENDFOR
    \STATE Pull arm $a_t = \arg\max_{a} \text{UCB}_a(t)$
    \STATE Observe reward $r_t$
    \STATE Update $\hat{\mu}_{a_t}(t)$ and $n_{a_t}(t)$
\ENDFOR
\end{algorithmic}
\end{algorithm}

\subsection{Thompson Sampling Algorithm}

Thompson Sampling is a probabilistic, Bayesian approach to the MAB problem. It maintains a posterior distribution over each arm's mean reward. At every round, a reward estimate is sampled from each posterior, and the arm with the highest sampled value is chosen.

\begin{enumerate}
  \item \textbf{Initialization:} Assume a prior distribution for the mean reward of each arm. For Bernoulli rewards, a natural choice is the Beta distribution, $Beta(\alpha_a, \beta_a)$, initialized with $\alpha_a = 1, \beta_a = 1$ (a uniform prior).
  
  \item \textbf{Loop at each time $t$:}
  \begin{itemize}
    \item For each arm $a$, sample a value $\theta_a(t)$ from its current posterior distribution $Beta(\alpha_a, \beta_a)$.
    \item Select the arm $a_t = \arg\max_{a} \theta_a(t)$.
    \item Observe the reward $r_t \in \{0, 1\}$.
    \item Update the posterior distribution for the chosen arm:
        \begin{equation}
        \alpha_{a_t} \leftarrow \alpha_{a_t} + r_t, \quad \beta_{a_t} \leftarrow \beta_{a_t} + (1 - r_t)
        \end{equation}
  \end{itemize}
\end{enumerate}

Intuitively, the sampling step automatically balances exploration and exploitation. An arm with high uncertainty (a wide posterior) has a higher chance of being sampled with a high value $\theta_a$, even if its current empirical mean is low. As an arm is pulled more often, its posterior distribution narrows around the true mean, and the sampled values become more consistent. For Thompson Sampling, a frequentist regret bound of $O(\ln T)$ has also been established.

\begin{algorithm}[h]
\caption{Thompson Sampling for Bernoulli Bandits}
\label{alg:thompson}
\begin{algorithmic}[1]
\STATE Initialize: $\alpha_a = 1$, $\beta_a = 1$ for all arms $a = 1, \dots, K$
\FOR{$t = 1, 2, \dots, T$}
    \FOR{each arm $a = 1, \dots, K$}
        \STATE Sample $\theta_a(t) \sim \text{Beta}(\alpha_a, \beta_a)$
    \ENDFOR
    \STATE Pull arm $a_t = \arg\max_{a} \theta_a(t)$
    \STATE Observe reward $r_t$
    \STATE Update: $\alpha_{a_t} \leftarrow \alpha_{a_t} + r_t$, $\beta_{a_t} \leftarrow \beta_{a_t} + (1 - r_t)$
\ENDFOR
\end{algorithmic}
\end{algorithm}

\subsection{Regret Analysis}

The regret analysis for both algorithms relies on bounding the number of times a suboptimal arm is selected. Let $\Delta_a = \mu^* - \mu_a$ be the suboptimality gap of arm $a$. For UCB, it can be shown that:

\begin{equation}
\mathbb{E}[n_a(T)] \leq \frac{8 \ln T}{\Delta_a^2} + O(1)
\end{equation}

which leads to the regret bound:
\begin{equation}
R_T \leq \sum_{a:\Delta_a > 0} \left( \frac{8 \ln T}{\Delta_a} + O(1) \right)
\end{equation}

For Thompson Sampling, a similar logarithmic regret bound can be established through Bayesian regret analysis or frequentist techniques, though the derivation is more involved due to the probabilistic nature of the algorithm.

\subsection{CAS Method Selection}

The Thompson Sampling approach assumes that rewards follow a fixed, known distribution (such as normal distribution or Beta distribution). However, in the fine-tuning process with multiple robustness, it is difficult to determine which specific probability distribution the rewards follow. The UCB algorithm avoids this issue and is more convenient in practice. Furthermore, their regret upper bounds are relatively similar. Therefore, in the CAS method, we have chosen the UCB algorithm to balance exploration and exploitation.

\section{Theoretical Details and Proofs}
\label{app: drift}

\paragraph{Setup and assumptions.}
Let $\mathcal{R}_p(\theta)$ and $\mathcal{R}_q(\theta)$ denote robust risks for two attack families $\mathcal{S}_p,\mathcal{S}_q$:
\begin{equation}
\mathcal{R}_i(\theta)=\mathbb{E}_{(x,y)\sim\mathcal{S}_i}\!\left[\max_{\delta\in\Delta_i}\mathcal{L}(f_\theta(x+\delta),y)\right],\quad i\in\{p,q\}.
\end{equation}
Define the average robust risk
\begin{equation}
\mathcal{R}_{avg}(\theta)=\tfrac{1}{2}\big(\mathcal{R}_p(\theta)+\mathcal{R}_q(\theta)\big)
\end{equation}
We make the following local regularity assumptions:
\begin{assumption}[Danskin and interchange]
For almost every $(x,y)$ the inner maximizer $\delta_i^*(\theta;x,y)$ is locally unique and sufficiently smooth in a neighborhood of $\theta_1$, so that the pointwise map
\(\ell_i(\theta;x,y):=\max_{\delta\in\Delta_i}\mathcal{L}(f_\theta(x+\delta),y)\)
is differentiable at $\theta_1$. Furthermore, there exists an integrable dominating function $g(x,y)$ such that $\|\nabla_\theta \ell_i(\theta;x,y)\|\le g(x,y)$ in a neighborhood of $\theta_1$. Under these conditions differentiation and expectation may be interchanged, yielding $
\nabla\mathcal{R}_i(\theta_1)=\mathbb{E}_{(x,y)\sim\mathcal{S}_i}\big[\nabla_\theta\ell_i(\theta_1;x,y)\big].$
\end{assumption}
\begin{assumption}[Second-order regularity] $\mathcal{R}_{avg}$ is twice continuously differentiable in a neighborhood of $\theta_1$, with Hessian $H_{avg}:=\nabla^2 \mathcal{R}_{avg}(\theta_1)$. Let $\lambda_{\max}(H_{avg})$ denote its largest eigenvalue at $\theta_1$.
\end{assumption}
\begin{assumption}[Positive spectral direction] We assume that the Hessian matrix $H_{avg}$ of $\mathcal{R}_{avg}$ at $\theta_1$ is not negative semi-definite, i.e., $\lambda_{\max}(H_{avg}) > 0$. Otherwise, $\mathcal{R}_{avg}$ would attain a local maximum at $\theta_1$. This would imply that, after adversarial training with respect to $\mathcal{R}_{q}$, the average robust risk over $S_p$ and $S_q$ reaches a local maximum with respect to the model parameters. Such a case is not only mathematically of measure zero but also contradicts the intuition of adversarial training.
\end{assumption}

\paragraph{Parameter drift formulation.}
Let the second fine-tuning phase move parameters by $\Delta\theta := \theta_2 - \theta_1$. In the small-step approximation we assume $\Delta\theta$ is aligned with the negative gradient of $\mathcal{R}_q$, i.e. $\Delta\theta= -\eta\,\nabla\mathcal{R}_q(\theta_1)$ for small $\eta>0$. Define the alignment angle $\psi$ between $\nabla\mathcal{R}_{avg}(\theta_1)$ and $\nabla\mathcal{R}_q(\theta_1)$ by
\begin{equation}
\nabla\mathcal{R}_{avg}(\theta_1)^\top \Delta\theta=-\|\nabla\mathcal{R}_{avg}(\theta_1)\|\,\|\Delta\theta\|cos\psi.
\end{equation}

\begin{lemma}[Average robust risk bound]
The change in average robust risk can be bounded as:
\begin{equation}
\Delta\mathcal{R}_{avg} \lesssim -\|\nabla\mathcal{R}_{avg}(\theta_1)\| \|\Delta\theta\| \cos\psi + \tfrac{1}{2}\lambda_{\max}(H_{avg}) \|\Delta\theta\|^2
\end{equation}
Where $\psi$ denotes the angle between $\nabla\mathcal{R}_{avg}(\theta_1)$ and $\nabla\mathcal{R}_{q}(\theta_1)$, $H_{avg}$ is the local Hessian of $\mathcal{R}_{avg}$, $\lambda_{\max}(H_{avg})$ denote its largest eigenvalue at $\theta_1$.
\end{lemma}
\begin{proof}
Expanding $\mathcal{R}_{avg}$ around $\theta_1$ gives
\begin{equation}
\begin{aligned}
\mathcal{R}_{avg}(\theta_2)-\mathcal{R}_{avg}(\theta_1)
= \nabla\mathcal{R}_{avg}(\theta_1)^\top \Delta\theta\\
+ \tfrac{1}{2}\,\Delta\theta^\top H_{avg}\,\Delta\theta
+ O(\|\Delta\theta\|^3).
\end{aligned}
\label{eq:appendix_taylor}
\end{equation}
Ignoring $O(\|\Delta\theta\|^3)$ for small steps, let
\begin{equation}
\mathcal{R}_{avg}(\theta_2)-\mathcal{R}_{avg}(\theta_1)\approx \nabla\mathcal{R}_{avg}(\theta_1)^\top \Delta\theta
+ \tfrac{1}{2}\,\Delta\theta^\top H_{avg}\,\Delta\theta
\end{equation}

Using the spectral upper bound $\Delta\theta^\top H_{avg}\Delta\theta \le \lambda_{\max}(H_{avg})\|\Delta\theta\|^2$, we obtain:
\begin{equation}
\Delta\mathcal{R}_{avg} \lesssim -\|\nabla\mathcal{R}_{avg}(\theta_1)\| \|\Delta\theta\| \cos\psi + \tfrac{1}{2}\lambda_{\max}(H_{avg}) \|\Delta\theta\|^2
\end{equation}
This scaling is conservative. If the update direction $\hat\Delta=\Delta\theta/\|\Delta\theta\|$ is known, the quadratic term can be tightened to $\tfrac12\,R(H_{avg},\hat\Delta)\|\Delta\theta\|^2$, where $R(H_{avg},\hat\Delta)$ is the Rayleigh quotient about $H_{avg}$ and the direction of $\hat\Delta$.
\end{proof}

\begin{proposition}[Safe parameter-drift threshold]
Let $\lambda_{\max}(H_{avg})$ denote the largest eigenvalue of the local Hessian of $\mathcal{R}_{avg}$ at $\theta_1$. A sufficient condition to prevent average robustness degradation after the sequential fine-tuning in the order of $S_p$ followed by $S_q$ is:
\begin{equation}
\|\Delta\theta\| < \frac{2 \|\nabla\mathcal{R}_{avg}(\theta_1)\| \cos\psi}{\lambda_{\max}(H_{avg})}, \quad \psi < \frac{\pi}{2}
\end{equation}
\end{proposition}
\begin{proof}
Based on Lemma~\ref{lemma}, a sufficient condition to ensure that the average robust risk continues to decrease under parameter drift is:
\begin{equation}
\begin{aligned}
\Delta\mathcal{R}_{avg} &\lesssim -\|\nabla\mathcal{R}_{avg}(\theta_1)\| \|\Delta\theta\| \cos\psi \\&+ \tfrac{1}{2}\lambda_{\max}(H_{avg}) \|\Delta\theta\|^2<0
\end{aligned}
\end{equation}
Assuming $H_{avg}(\theta_1)$ is not negative semi-definite, so that $\lambda_{\max}(H_{avg})>0$, solving the inequality for $\|\Delta\theta\|$ yields the safe threshold:
\begin{equation}
\|\Delta\theta\| < \frac{2 \|\nabla\mathcal{R}_{avg}(\theta_1)\| \cos\psi}{\lambda_{\max}(H_{avg})}
\end{equation}
Hence, if $\|\Delta\theta\|$ exceeds this bound, sequential fine-tuning may increase $\mathcal{R}_{avg}$, i.e., reduces average robustness. 
$\square$
\end{proof}

\paragraph{Remarks.}
\begin{itemize}
  \item The bound in Proposition~\ref{prop} is sufficient but not necessary. It is conservative because of the spectral upper bound used on the Hessian term; a tighter estimate of the Rayleigh quotient yields a less conservative threshold.
  \item In practice the quantities $\|\nabla\mathcal{R}_{avg}\|$, $\cos\psi$ and $\lambda_{\max}(H_{avg})$ may be estimated from mini-batches and Hessian-vector products (e.g. via power iteration / Pearlmutter products). Because of sampling noise these estimates should be treated as guidance rather than absolute thresholds.
\end{itemize}

\section{Multi-Objective Optimization Perspective on Cross-Type Trade-off}
\label{app: external}
\subsection{Externality-Aware Equilibrium}

When training proceeds for a sufficiently long period, the uncertainty in reward estimation diminishes as each perturbation type has been sampled enough times. In this asymptotic regime, the influence of the exploration term in the UCB score becomes negligible, since
\begin{equation}
\lim_{t \to \infty}\sqrt{\frac{2 \log N}{N_v}} \;\xrightarrow\; 0,
\end{equation}
and the sampling probability in Eq.~(\ref{eq:sampling}) simplifies to a softmax-augmented exploitation form:
\begin{equation}
\pi_v \;\propto\; w_v \cdot \exp(\alpha \cdot R_v),
\label{eq:equilibrium_prob}
\end{equation}
where $R_v = R_v^{\text{self}} + R_v^{\text{tradeoff}}$ includes both marginal and cross-type effects.  

This asymptotic distribution represents an \emph{externality-aware equilibrium}, in which the long-term sampling behavior aligns with the relative overall contribution of each perturbation type to system-wide robustness. Perturbations that yield positive spillover effects naturally gain higher long-run probabilities, while those producing negative externalities are suppressed.  

Consequently, the CAS algorithm converges to a stable allocation on the Pareto frontier, where the equilibrium mixture of perturbation types reflects not only their intrinsic robustness importance ($w_v$) but also their dynamic rewards ($R_v$). This leads to a self-organizing and globally balanced fine-tuning process.

\subsection{Pareto-Frontier Perspective on Cross-Type Trade-off}

From a multi-objective optimization perspective, the process of adversarial fine-tuning across multiple perturbation types can be viewed as searching along a Pareto frontier of robustness objectives $\{\mathcal{R}_{p_1}, \ldots, \mathcal{R}_{p_M}\}$. Each objective corresponds to robustness under one adversarial type, and jointly optimizing them inevitably involves trade-offs.  

In the standard setting, the overall optimization direction is determined by the weighted gradient aggregation
\begin{equation}
\nabla_\theta \mathcal{R}_{\text{avg}} = \sum_{k=1}^M w_k \nabla_\theta \mathcal{R}_{p_k},
\end{equation}
where $w_k$ reflects the static relative importance of each perturbation type. This aggregation assumes that improvements in one objective do not introduce externalities to others, which is violated in practice. For instance, improving robustness against $\ell_\infty$ attacks may slightly degrade robustness against blur or color-shift perturbations, forming a locally concave Pareto surface.  

To explicitly account for such cross-type interactions, the CAS method introduces a cross-type trade-off term into the reward formulation. This mechanism dynamically adjusts the effective sampling weights according to each type’s net contribution (positive or negative) to the global robustness improvement. Geometrically, this modification changes the optimization trajectory from a fixed-weight combination 
\begin{equation}
\sum_k w_k \nabla_\theta \mathcal{R}_{p_k}
\end{equation}
to a reweighted direction 
\begin{equation}
\sum_k w_k \nabla_\theta \mathcal{R}_{p_k}*\exp(\alpha R_k),
\end{equation}
that effectively steering the update vector toward regions of the Pareto frontier with higher joint gains. Instead of optimizing only along a pre-defined linear scalarization of the objectives, CAS adaptively reshapes the frontier traversal according to dynamic inter-type feedback, allowing the fine-tuning process to converge toward more globally balanced robustness equilibria.  

In summary, introducing cross-type trade-off terms transforms the optimization landscape from a static scalarized objective into a \textit{dynamically calibrated Pareto search}. By rewarding perturbation types that produce positive spillover effects and penalizing those inducing interference, CAS achieves more equitable and globally efficient convergence across heterogeneous robustness objectives.

\subsection{Economic Motivation Behind Pareto Frontier and Reward Design in CAS}
\label{subsec:econ_motivation}

The concept of the Pareto frontier in multi-objective learning is directly inspired by the notion of \emph{Pareto efficiency} in economics, which characterizes an allocation where no objective can be improved without worsening another. In the context of multi-robustness optimization, each perturbation type can be regarded as an ``agent'' competing for model capacity—a limited resource analogous to economic capital or utility. The CAS framework extends this analogy by introducing \emph{externality-aware rewards}, effectively capturing how the training on one perturbation type influences others. 

This heuristic motivation underscores a deep parallel: In the presence of externalities, the free market mechanism cannot automatically achieve Pareto optimality, and government or institutional intervention becomes necessary. According to the Coase theorem, the problem of externalities does not lie in the existence of costs themselves, but in the failure to assign these costs correctly, that is, to internalize them. We extend this idea to multi-robustness fine-tuning, by internalizing the impact of each adversarial training component on other robustness dimensions within the reward design. From an economic standpoint, the inclusion of cross-type trade-off terms resembles a system of \emph{Pigouvian adjustments}, where positive externalities are rewarded and negative ones are penalized to achieve social welfare optimality. Here, the ``social welfare'' corresponds to the collective robustness performance across all perturbation types, and CAS dynamically regulates training probabilities to approximate this welfare-optimal equilibrium.

\section{Empirical Analysis of Trade-off Matrix}
\label{app:multi-robustness}
\begin{figure*}[ht]
    \centering
    \includegraphics[width=1\linewidth]{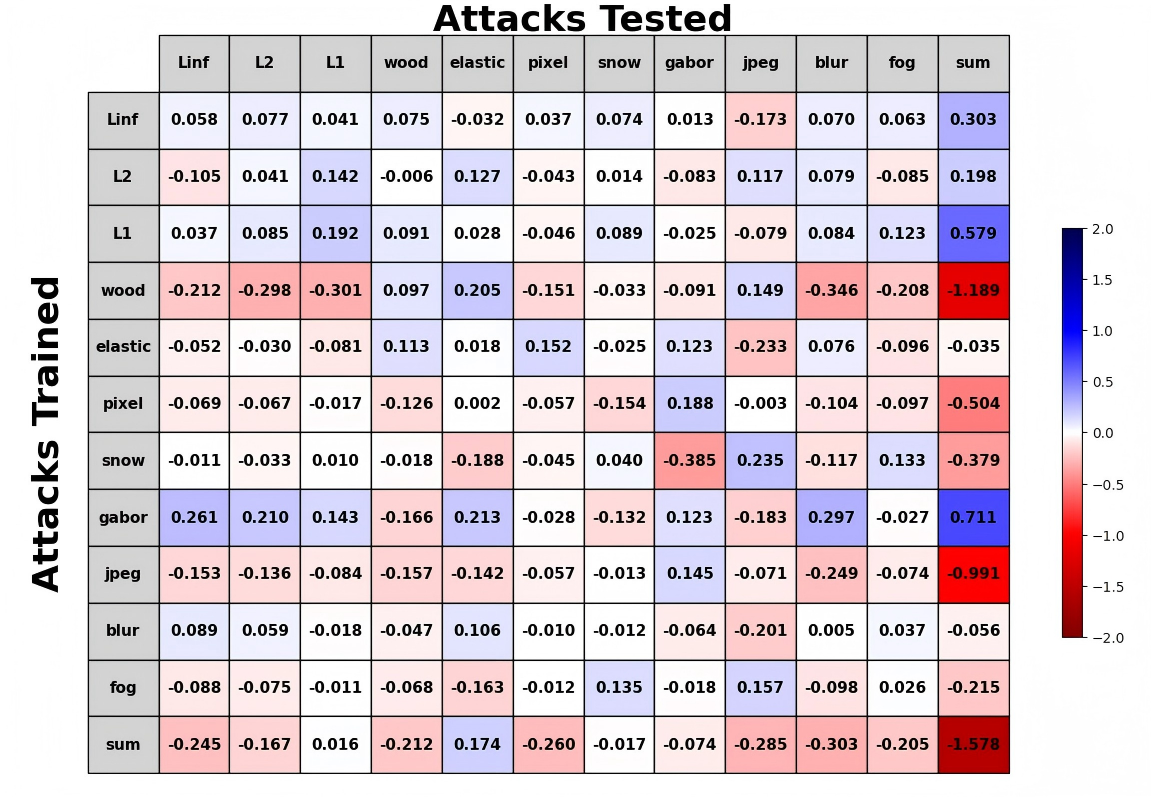}
    \caption{Detailed visualization of the trade-off matrix. In the figure, each number represents the change in robust accuracy for its corresponding attack type after the sequential fine-tuning against designated types. Each row shows the data fine-tuned against the adversarial attack labeled on the left; each column shows the data evaluated with the attack labeled above.}
    \label{app: tradeoff}
\end{figure*}
The trade-off matrix presented in Figure~\ref{app: tradeoff} provides a comprehensive quantitative view of the complex interactions between different adversarial perturbation types during sequential fine-tuning. This analysis reveals critical insights into the Pareto-optimal navigation strategy employed by our CAS method.

\subsection{Matrix Structure and Interpretation}
The trade-off matrix establishes a systematic framework for understanding how training on one perturbation type influences robustness against others. Each entry represents the directional effect between perturbation pairs, where positive values indicate beneficial transfer, negative values reveal detrimental interference, and near-zero values suggest independence. Diagonal entries are mostly positive, confirming that attack-specific training enhances self-robustness, while off-diagonal values reveal substantial trade-offs. (The phenomenon that some adversarial training even reduces its own robustness may be due to the loss landscape being unsmoothed at the early stage of fine-tuning, leading to unstable robustness. Increasing the number of fine-tuning epochs for each attack type in sequential fine-tuning might help prevent this issue.) The row and column sums serve as aggregate indicators of each perturbation type's overall impact on the multi-objective robustness landscape. 

\subsection{Dominant Positive Transfer Patterns}
Analysis reveals that certain perturbation types serve as powerful catalysts for broad robustness improvements. Gabor filter attacks emerge as particularly effective, demonstrating widespread positive transfer across multiple perturbation families including norm-based attacks and elastic transformations. This suggests that Gabor training induces features with generalizable robustness properties. Similarly, L1 attacks show strong self-improvement coupled with positive spillover effects, particularly benefiting geometric and weather-based perturbations. These patterns indicate that carefully selected perturbation types can transcend their specific domains to enhance overall model resilience, providing valuable guidance for efficient multi-objective optimization.

\subsection{Problematic Negative Transfer Cases}
Conversely, several perturbation types exhibit consistently detrimental effects on overall robustness. JPEG compression demonstrates the most severe negative transfer. Wood perturbations similarly show widespread interference, particularly compromising robustness against common corruptions. These findings highlight the critical importance of identifying and mitigating such negative transfer effects, as indiscriminate training on problematic perturbation types can undermine the entire multi-robustness objective through destructive interference.

\subsection{Asymmetric Transfer Relationships}
The matrix reveals intriguing asymmetric relationships where transfer effects differ dramatically based on directionality. For instance, while Linf training moderately benefits L2 robustness, the reverse relationship proves detrimental, with L2 training actually degrading Linf performance. Similar asymmetries appear between elastic and snow perturbations, where snow training has disproportionately negative effects on elastic robustness compared to the reciprocal relationship. These directional dependencies underscore the limitations of symmetric treatment in multi-perturbation training and emphasize the need for carefully calibrated, direction-aware sampling strategies.

\subsection{Implications for Reward Design}
The empirical patterns observed in the trade-off matrix provide strong validation for our cross-type reward design. The systematic quantification of transfer effects enables intelligent navigation toward Pareto-optimal regions where positive transfer dominates, while avoiding conflict zones where robustness objectives severely compete. By dynamically incorporating these relationships into the reward mechanism, our approach naturally favors perturbation types that serve as robustness catalysts while penalizing those causing destructive interference. This data-driven strategy transforms the complex multi-objective optimization problem into a tractable navigation task, where the trade-off matrix serves as both validation of our approach and guidance for its continuous refinement.

\section{Detailed Experimental Results and Supplementary Analysis}
\label{app: ablation}
\subsection{Module Contribution Ablation} 
As shown in Table~\ref{tab: component}, removing $R^{\text{tradeoff}}$ consistently leads to a noticeable performance drop across all datasets, especially in terms of robustness metrics. This indicates that $R^{\text{tradeoff}}$ plays a crucial role in mitigating cross-type trade-offs by explicitly guiding the optimization toward a better Pareto frontier. In contrast, the removal of UCB shows only a slight degradation in robustness, while clean accuracy remains nearly unchanged. This suggests that UCB mainly stabilizes training by adaptively adjusting the sampling of attack types according to their estimated uncertainty, rather than directly contributing large numerical gains. Therefore, $R^{\text{tradeoff}}$ primarily governs the effectiveness of robustness trade-off learning, while UCB enhances training stability and ensures consistent convergence across diverse attack distributions.
\begin{table*}[htbp]
\centering
\caption{Ablation study of our CAS method about the component of $R^{\text{tradeoff}}$ and UCB. Both the definitions and the calculation of \textbf{avg. $\ell_p$} and \textbf{avg. Corruption} are the same as in the main experiment. The table reports the average results over five independent runs on three datasets.}
\begin{tabular}{c|c|cccc}
\toprule
\textbf{Dataset} & \textbf{Method} & \textbf{Clean} & \textbf{avg. Robust} & \textbf{avg. $\ell_p$} & \textbf{avg. Semantic} \\
\midrule
\multirow{3}*{CIFAR-10}& CAS & \textbf{85.26} & \textbf{51.79} & \textbf{52.91} & \textbf{50.67}\\
& CAS w/o $R^{\text{tradeoff}}$ & 84.40& 51.42& 52.70& 50.14\\
& CAS w/o UCB & 85.08 & 51.65 & 52.80 & 50.51\\
\midrule
\multirow{3}*{CIFAR-100}& CAS & 58.54 & \textbf{27.99} & \textbf{29.23} & 26.75\\
& CAS w/o $R^{\text{tradeoff}}$ & 58.00& 27.54 & 28.77& 26.32\\
& CAS w/o UCB & \textbf{58.68} & 27.76 & 28.50 & \textbf{27.01}\\
\midrule
\textbf{Dataset} & \textbf{Method} & \textbf{Clean} & \textbf{avg. Robust} & \textbf{avg. Corruption} & \textbf{avg. Others}\\
\midrule
\multirow{3}*{SVHN}& CAS & 93.86 & \textbf{57.76} & \textbf{63.81} & \textbf{51.71}\\
& CAS w/o $R^{\text{tradeoff}}$ & \textbf{94.24}& 55.77& 62.59& 48.94\\
& CAS w/o UCB & 93.80 & 57.59 & 63.52 & 51.66\\
\bottomrule
\end{tabular}
\label{tab: component}
\end{table*}
\begin{figure*}[htbp]
    \centering
    \includegraphics[width=1\linewidth]{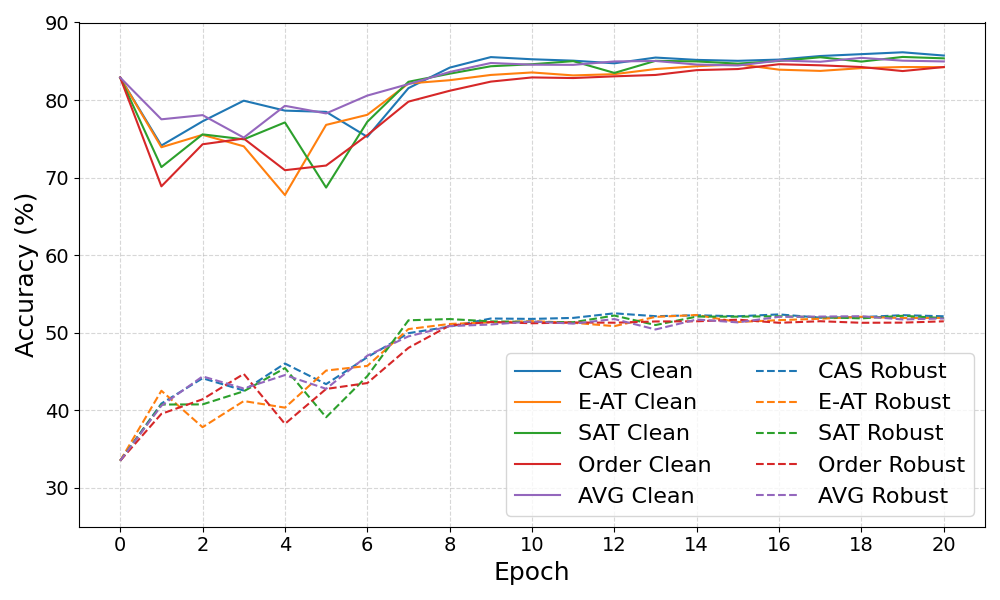}
    \caption{Ablation analysis on training epochs, conducted on the CIFAR-10 dataset. Robust accuracy is the weighted average; all accuracy values are means of five independent runs.}
    \label{fig:cifar10}
\end{figure*}
\begin{figure*}[htbp]
    \centering
    \includegraphics[width=1\linewidth]{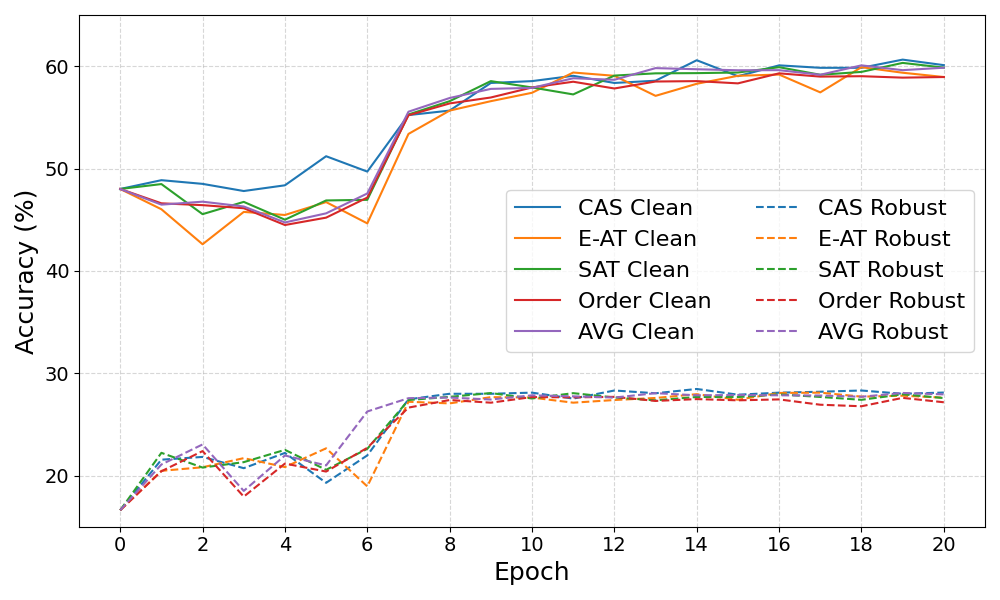}
    \caption{Ablation analysis on training epochs, conducted on the CIFAR-100 dataset. Robust accuracy is the weighted average; all accuracy values are means of five independent runs.}
    \label{fig:cifar100}
\end{figure*}
\begin{table*}[htbp]
\centering
\small
\caption{Sensitivity analysis on hyperparameters $\alpha$ and $\beta$, where $\alpha$ is used to balance exploration and exploitation, and $\beta$ is used to balance clean accuracy and robustness. The results in the table are averaged over five independent runs.}
\begin{tabular}{c|c|cc|c|cc||c|ccc|c|c}
\toprule
\textbf{Dataset} & \textbf{$\beta=\frac89,\alpha=?$} & \textbf{1} & \textbf{5}& \textbf{10} & \textbf{20} & \textbf{50} & \textbf{$\alpha=10,\beta=?$} & $\frac{\textbf{1}}{\textbf{2}}$& $\frac{\textbf{2}}{\textbf{3}}$ & $\frac{\textbf{4}}{\textbf{5}}$ & $\frac{\textbf{8}}{\textbf{9}}$ & $\frac{\textbf{16}}{\textbf{17}}$\\
\midrule
\multirow{4}*{CIFAR-10} &Clean         & 84.32 & 84.68 & 85.26 & \textbf{85.42} & 85.04 &Clean& \textbf{86.02} & 85.76 & 85.36 & 85.26 & 85.12\\
&avg. Robust   & 51.68 & 51.50 & \textbf{51.79} & 51.55 & 50.83 &avg. Robust& 50.42 & 50.96 & 51.45 & 51.79 & \textbf{51.81}\\
&avg. $\ell_p$ & 53.10 & \textbf{53.27} & 52.91 & 52.83 & 50.60 &avg. $\ell_p$& 51.19 & 51.99 & 52.36 & \textbf{52.91} & 52.90\\
&avg. Semantic & 50.26 & 49.73 & 50.67 & 50.26 & \textbf{51.05} &avg. Semantic& 49.66 & 49.93 & 50.55 & 50.67 & \textbf{50.73}\\
\midrule
\multirow{4}*{CIFAR-100} &Clean         & 58.46 & \textbf{58.62} & 58.54 & 58.12 & 57.48 &Clean& \textbf{58.88} & 58.74 & 58.50 & 58.54 & 58.46\\
&avg. Robust   & 27.68 & \textbf{28.05} & 27.99 & 27.88 & 27.61 &avg. Robust& 27.46 & 27.77 & 28.01 & 27.99 & \textbf{28.08}\\
&avg. $\ell_p$ & 28.89 & \textbf{29.37} & 29.23 & 28.87 & 28.03 &avg. $\ell_p$& 28.74 & 29.00 & \textbf{29.30} & 29.23 & 29.29\\
&avg. Semantic & 26.47 & 26.72 & 26.75 & 26.89 & \textbf{27.19} &avg. Semantic& 26.28 & 26.53 & 26.71 & 26.75 & \textbf{26.87}\\
\midrule
\multirow{4}*{SVHN} &Clean         & 93.80 & \textbf{93.98} & 93.86 & 93.58 & 92.98 &Clean& \textbf{94.03} & 93.87 & 93.89 & 93.86 & 93.82\\
&avg. Robust   & 57.45 & 57.73 & \textbf{57.76} & 57.62 & 57.17 &avg. Robust& 57.41 & 57.61 & 57.68 & 57.76 & \textbf{57.90}\\
&avg. Corruption & 63.57 & \textbf{63.81} & \textbf{63.81} & 63.48 & 62.36 &avg. Corruption& 63.35 & 63.70 & 63.80 & 63.81 & \textbf{63.84}\\
&avg. Others & 51.32 & 51.64 & 51.71 & 51.77 & \textbf{51.99} &avg. Others& 51.47 & 51.53 & 51.55 & 51.71 & \textbf{51.96}\\
\bottomrule
\end{tabular}
\label{tab: hyper}
\end{table*}
\begin{table*}[htbp]
\centering
\caption{Generalization comparison on CIFAR-10 MultiRobustBench with different architectures. $CR_{avg}$ denotes the average competitiveness ratio. Baseline results are obtained from the official leaderboard, while CAS results are reproduced using our released implementation.}
\label{tab:mrb_cifar10_appendix}
\begin{tabular}{l|c|ccc}
\toprule
Architecture & Method & Clean & $CR_{avg}$ & PFLOPs \\
\midrule
ResNet-18
& Manifold Regularization / CAS
& 72.14 / \textbf{73.80}
& 58.74 / \textbf{61.87}
& 39.00 / \textbf{28.27}
\\
PreAct-ResNet18
& E-AT / CAS
& 82.41 / \textbf{83.62}
& \textbf{52.17} / 51.89
& \textbf{6.0} / 6.2
\\
WRN-28-10
& MNG-AC / CAS
& \textbf{81.45} / 80.73
& \textbf{63.54} / 61.05
& 621.0 / \textbf{52.4}
\\
ResNet-50
& PAT (Fast LPA) / CAS
& 71.58 / \textbf{72.46}
& \textbf{58.49} / 57.17
& 457.0 / \textbf{94.8}
\\
\bottomrule
\end{tabular}
\end{table*}

\subsection{Epochs} 
Figure~\ref{fig:cifar10} and Figure~\ref{fig:cifar100} illustrates the evolution of clean and robust accuracy during fine-tuning on the CIFAR-10 and CIFAR-100 datasets across different adversarial training schemes. For CIFAR-10, robust accuracy rapidly improves within the first 10 epochs, after which performance stabilizes. The clean accuracy fluctuates in the early stage and then converges to a stable level, indicating that moderate fine-tuning epochs are sufficient to achieve a balance between clean and robust performance. In contrast, for CIFAR-100, the overall accuracy levels are lower due to the increased task difficulty. Clean accuracy rises gradually from 45\% to around 60\%, whereas robust accuracy improves from 15\% to about 28\% over 20 epochs, showing slower convergence and higher sensitivity to training instability. In particular, the clean accuracy still exhibits unstable fluctuations even after 20 epochs. These results suggest that a longer fine-tuning schedule benefits datasets with higher complexity (e.g., CIFAR-100), where the loss landscape is more irregular and the optimization process requires more iterations to stabilize. Conversely, excessive fine-tuning on simpler datasets (e.g., CIFAR-10) yields diminishing returns and may even introduce mild overfitting. Overall, Our choice of 10 epochs achieves a good balance between time cost and performance across different datasets.

\subsection{Hyperparameters} 
Table~\ref{tab: hyper} presents the ablation results of the hyperparameters $\alpha$ and $\beta$, which respectively control the trade-offs between exploration–exploitation and clean–robust performance. When varying $\alpha$ under a fixed $\beta=\tfrac{8}{9}$, we observe that moderate values (around $\alpha=5$--$10$) consistently yield a good balance between clean and robust accuracy across datasets. Smaller $\alpha$ leads to underexploitation, limiting the improvement of robustness, whereas excessively large $\alpha$ overemphasizes exploitation and harms clean accuracy and robustness due to unstable parameter updates. In contrast, adjusting $\beta$ while fixing $\alpha=10$ reveals that increasing $\beta$ gradually shifts the model’s focus toward robustness, improving average robust accuracy but with sacrifices in clean performance.

This trend is more evident on CIFAR-10, where robustness gains saturate beyond $\beta=\tfrac{8}{9}$, suggesting that too strong a robustness constraint reduces optimization flexibility. However, different dataset tasks exhibit distinct preferences for the hyperparameters. For instance, CIFAR-10 favors larger values of $\alpha$ (around 10–20), while CIFAR-100 and SVHN achieve better performance when $\alpha$ is set between 5 and 10. A value of $\beta = \frac{8}{9}$ provides a good balance between clean accuracy and robustness for CIFAR-10 and CIFAR-100, whereas for SVHN, clean accuracy saturates more easily, and using a larger $\beta$ can further enhance robustness with relatively minor impact on clean accuracy. Across all datasets, the best configurations emerge when $\alpha$ and $\beta$ are jointly balanced, supporting our design intuition that moderate exploration and a proportional robustness weighting lead to stable and well-generalized models.

\subsection{Generalization across Architectures and Robustness Benchmarks}
To evaluate the generalization capability of CAS beyond the architectures and settings considered in the main experiments, we further conduct experiments on CIFAR-10 MultiRobustBench. We compare CAS with representative multi-robust training baselines under different architectures, including ResNet-18, PreAct-ResNet18, WRN-28-10, and ResNet-50. All CAS results are obtained using the same training and evaluation configurations released with our method, while baseline results~\cite{jin2020manifold, croce2022adversarial, madaan2021learning, laidlaw2021perceptualadversarialrobustnessdefense} are collected from the official MultiRobustBench~\cite{dai2023multirobustbench} leaderboard.

Table~\ref{tab:mrb_cifar10_appendix} reports the clean accuracy, average competitiveness ratio ($CR_{avg}$), and computational cost measured by PFLOPs. CAS consistently achieves competitive or superior robustness performance across different architectures. In particular, CAS improves $CR_{avg}$ on ResNet-18 and achieves comparable robustness on other architectures while substantially reducing computational overhead in several cases. Furthermore, MultiRobustBench aggregates robustness performance under different perturbation strengths to obtain the final score, providing additional evidence that our method generalizes well across different levels of perturbations. These results indicate that the adaptive attack sampling strategy is not tied to a specific network architecture and can generalize across different model capacities and training frameworks.

\subsection{Sensitivity Analysis of Attack Weights}
To further investigate the robustness of CAS under different attack weights choices, we conduct additional sensitivity experiments on CIFAR-10. We study the influence of the weight assigned to $\ell_p$ attacks while fixing the weights of all 18 semantic attacks to 1. As shown in Table~\ref{tab:weight_sensitivity_appendix}, varying the $\ell_p$ weight changes the balance between $\ell_p$ robustness and semantic robustness, but the overall robust accuracy remains stable. A larger $\ell_p$ weight improves robustness against norm-based attacks, while reducing the emphasis on semantic transformations. The default setting achieves a balanced trade-off between different attack categories.
\begin{table}[t]
\centering
\caption{Sensitivity analysis of the $\ell_p$ attack weight on CIFAR-10. Semantic attack weights are fixed to 1.}
\label{tab:weight_sensitivity_appendix}
\begin{tabular}{c|cccc}
\toprule
$\ell_p$ weight 
& 1 & 3 & 6 & 12 \\
\midrule
Clean 
& 84.10 & 84.80 & \textbf{85.26} & 84.50 \\
Avg. Robust
& \textbf{52.59} & 52.54 & 51.79 & 51.95 \\
Avg. $\ell_p$
& 50.37 & 52.57 & 52.91 & \textbf{53.18} \\
Avg. Semantic
& \textbf{52.96} & 52.52 & 50.67 & 49.19 \\
\bottomrule
\end{tabular}
\end{table}



\section{Promising Research Directions}
Finally, we reflect on the study's limitations and outline promising research directions:

1. Advancing trade-off characterization: Rigorous benchmarking frameworks or novel theoretical frameworks for multi-robustness trade-offs are waiting to be developed.

2. Expanding attack typology: New adversarial constraints could be proposed, or existing semantic attacks refined through granular categorization to clarify robustness conflicts/synergies.

3. Exploring composite perturbations: Stronger attacks may emerge from strategically combined perturbation types.

4. Formalizing non-algorithmic attacks: Attacks lacking clear algorithmic definitions warrant alternative formalizations or methodological innovations for systematic study.

\end{document}